\documentclass{article}
\usepackage{PRIMEarxiv}
\usepackage[utf8]{inputenc}
\usepackage[T1]{fontenc}
\usepackage{hyperref}   
\usepackage{url}        
\usepackage{booktabs}   
\usepackage{amsfonts}   
\usepackage{nicefrac}   
\usepackage{microtype}  
\usepackage{lipsum}
\usepackage{fancyhdr}   
\usepackage{graphicx}   
\title{ Chain-of-Thought Shows the Path to a Tree: Realizing Branching Complexity }

\author{
  Debanjan Dutta \\
  Indian Statistical Institute \\
  Kolkata, India \\
   \And
  Anish Chakrabarty\\
  LTCI, Télécom Paris\\
  Palaiseau, France
   \And
  Swagatam Das\textsuperscript{$\dagger$} \\
  Indian Statistical Institute \\
  Kolkata, India
}
\usepackage{natbib}

\hypersetup{colorlinks=true,
}
\usepackage{amsmath, amsthm, amsfonts, amssymb, bbm, pifont, pdfpages, hyperref, cleveref, enumitem, wrapfig, mathtools}
\usepackage[dvipsnames]{xcolor}

\newtheorem{theorem}{Theorem} 

\usepackage{aliascnt}
\newaliascnt{lemma}{theorem}
\newtheorem{lemma}[lemma]{Lemma}
\aliascntresetthe{lemma}

\newaliascnt{proposition}{theorem}
\newtheorem{proposition}[proposition]{Proposition}
\aliascntresetthe{proposition}

\theoremstyle{remark}
\newtheorem{remark}{Remark}[theorem]

\makeatletter
\def\th@remark{
  \thm@headfont{\bfseries\itshape} 
  \thm@notefont{\normalfont\itshape}         
  \normalfont                      
}
\makeatletter
\def\th@observation{
  \thm@headfont{\bfseries\itshape} 
  \thm@notefont{\normalfont\itshape} 
  \normalfont                           
}
\makeatother
\theoremstyle{observation}
\newtheorem{observation}{Observation}

\makeatletter
\def\th@fact{
  \thm@headfont{\bfseries\itshape} 
  \thm@notefont{\normalfont\itshape} 
  \normalfont                           
}
\makeatother
\theoremstyle{fact}
\newtheorem{fact}{Fact}

\theoremstyle{definition}
\newtheorem{definition}{Definition}[theorem]

\newtheorem*{reptheorem}{\autoref{\currentreflabel}}
\newenvironment{restate}[1]
  {\def\currentreflabel{#1}\begin{reptheorem}}
  {\end{reptheorem}}
\newtheorem*{replemma}{\autoref{\currentreflabel}}

\newcommand{\R}{\mathbb{R}}
\newcommand{\Z}{\mathbb{Z}}
\newcommand{\dij}{\mathsf{Dkst}}
\newcommand{\dfs}{\mathsf{dfs}}
\newcommand{\Dfs}{\mathsf{DFS}}
\newcommand{\dis}{\mathsf{Dist}}
\newcommand{\nei}{\mathcal{N}}
\DeclareMathOperator{\argmin}{argmin}

\DeclareMathOperator{\relu}{ReLU}
\DeclareMathOperator{\ind}{\mathbbm{1}}

\DeclareMathOperator{\hgt}{ht}
\DeclareMathOperator{\wid}{wd}
\DeclareMathOperator{\str}{\mathsf{st}}
\newcommand{\dep}{\mathrm{depth}}
\newcommand{\UU}{\mathtt{U}}
\newcommand{\DD}{\mathtt{D}}

\usepackage[ruled,linesnumbered]{algorithm2e}
\makeatletter
\renewcommand{\algocf@float}{false}
\makeatother

\usepackage{etoolbox}
\makeatletter
\patchcmd{\@algocf@start}{-1.5em}{0em}{}{}
\makeatother
\makeatletter
\SetAlgoCaptionSeparator{.}

\newcommand{\thmautoref}[1]{\hyperref[#1]{Thm.~\ref*{#1}}}

\usepackage{tikz}
\usetikzlibrary{arrows.meta}
\usetikzlibrary{arrows.meta, calc}

\definecolor{colA}{HTML}{EF4F4F}   
\definecolor{colB}{HTML}{2F8FEF}   
\definecolor{colC}{HTML}{1FAF8F}   
\definecolor{colD}{HTML}{E8820C}   
\definecolor{colE}{HTML}{8B5CF6}   
\definecolor{axpink}{HTML}{D4537E}
\definecolor{pathteal}{HTML}{1D9E75}
\definecolor{darkblue}{rgb}{0, 0, 0.5}

\tikzset{tnode/.style={circle, inner sep=0pt, minimum size=4.5pt}}

\hypersetup{
        pdfinfo={
            Author={Debanjan Dutta, Anish Chakrabarty, Swagatam Das},
            Title={ Chain-of-Thought Shows the Path to a Tree: Realizing Branching Complexity },
            Subject={arXiving 2026 Work},
            Keywords={CoT, Branching Complexity, Strahler Number},
            CreationDate={10/08/2026--01:34:00},
            ModDate={\the\day/\the\month/\the\year\space},
            Creator={PdfLaTeX v3.14},
            Producer={LaTeX},
        },
        colorlinks=true,
        allcolors=darkblue,
        citecolor=Aquamarine,
    }

\begin{document}
\maketitle
{%
  \renewcommand{\thefootnote}{}%
  \footnotetext{\textsuperscript{$\dagger$}Corresponding author email: \href{mailto:swagatam.das@isical.ac.in?subject=[From arXiv] CoTRank Paper}{\texttt{swagatam.das@isical.ac.in}}}%
}

\begin{abstract}
    Chain of Thought (CoT) lifts the expressive ceiling of bounded-depth Transformers, with characterizations tying the number of CoT steps to circuit complexity classes. What remains largely missing are concrete instantiations with explicit, depth-bounded constructions, and the traversal procedures such characterizations presuppose. We close this gap for branching complexity. We give CoT realizations of depth-first search (DFS) and of Dijkstra algorithm, the latter subsuming breadth-first search, by unique hard-attention decoders of at most two layers, and use them as a shared computational substrate: reusing the DFS decoder yields the Strahler number of an $n$-vertex tree in $2n-1$ steps with four layers, and reusing the Dijkstra decoder yields its width in $n-1$ steps with three. Since computing the Strahler number of a binary tree given as a term is \textsf{NC\textsuperscript{1}}-complete, and our constructions handle arbitrary $n$-ary trees without layer normalization or positional encodings, this is a non-trivial witness for the linear-step regime of the CoT hierarchy. Exploiting the classical bijection between ordered trees and Dyck paths, itself realized by our DFS construction, which emits the path as it traverses, we give independent constructions for both measures on the path representation.
\end{abstract}

\section{Introduction}
\label{sec:intro}
    Chain-of-thought (CoT) prompting has turned autoregressive Transformers into systems that appear to execute procedures rather than merely predict tokens \citep{wei2022chain}. Going beyond empirical curiosity, \citet{merrill2024expressive} show that CoT lets bounded-depth Transformers escape the complexity-class ceiling of a single forward pass, and \citet{barcelo2025ehrenfeuchthaussler} go further, giving an exact fixed-step characterization in terms of the Ehrenfeucht–Haussler (EH) rank of a function. Such \textit{possibility} results still allow for explicit, minimal-depth CoT construction solving non-trivial problems exhibiting the traversal or recursion, which are only known to exist. \citet{barcelo2025ehrenfeuchthaussler} themselves note that their bound is achieved through exhaustive tree traversal, but they do not construct the traversal. \citet{zhu2026emergence} study whether continuous CoT can solve reachability, but only empirically. Our work begins with this fact in mind.

    Branching complexity of a tree \textemdash measured classically by the Strahler number \citep{strahler1957quantitative} and, independently, by the width of a tree \textemdash is one of the oldest quantitative questions about hierarchical structure, with consequences ranging from the minimum number of registers needed to evaluate an arithmetic expression \citep{flajolet1979number} to bounds on random tree shape \citep{addario2013subgausian} to the size of the largest antichain \citep{dilworth1950decomposition}. Moreover, it turns out to be readily tied to the CoT lore. The Strahler number of binary trees and the EH rank on binary decision trees satisfy the same recursive update \citep{dahiya2021onsimple}, and \citet{ganardi2026complexity} have just shown that computing the Strahler number of a binary tree is \textsf{NC\textsuperscript{1}}-complete {--} precisely the class \citet{merrill2024expressive} associate with a \textit{linear} number of CoT steps. In simpler terms, we now know of a rank measure known to sit at the boundary of what constant (or, linear) CoT depth can express, and a general theorem stating that such measures should be expressible with the right traversal machinery. In this work, we introduce the machinery.

    We give explicit constructions of depth-first search (DFS), breadth-first search (BFS, as a special case of \citeauthor{dijkstra1959note}'s algorithm), and Dijkstra, realized by Transformer decoders with two layers and two attention heads under CoT. The findings do not emerge as isolated expressivity results but are based on a \textit{computational substrate} on which branching complexity becomes CoT-accessible. For example, reusing the same DFS realization that produces the traversal order, we compute the Strahler number in $2n-1$ steps for an $n$-vertex tree (\autoref{thm:cottreestrahler}), and reusing the Dijkstra realization, we compute the width of a tree in $n-1$ steps (\autoref{thm:cotwidoftree}). Because computing the Strahler number of a general $n$-ary tree is a strict generalization of the binary, tree setting in which the \textsf{NC\textsuperscript{1}}-completeness result was established, this is a non-trivial concrete witness for the \citeauthor{merrill2024expressive} framework, obtained without appealing to auxiliary primitives such as layer normalization that prior CoT constructions have relied on.

    The traversal machinery gives us a second, independent route to the same measures. The classical bijection between ordered trees and Dyck paths \citep{goldman1992lattice} allows us to flatten a tree onto a one-dimensional lattice path while providing an algorithmic witness of the bijection. The fact that the number of full binary trees $G_T$ with $n$ internal vertices and Strahler number $\str(G_T)$ turns out to be the same as the number of Dyck paths of length $2n$ whose height $h$ satisfies $\lfloor \log_2(1+h)\rfloor = \str(G_T)$ \citep{flajolet1979number,addarioberry2024refinedhortonstrahlernumbersi} becomes crucial in context. The DFS realization of \autoref{thm:cotdfs} \emph{is} the map from tree to path, since a traverse step is exactly an up-step and a backtrack step is exactly a down-step. This raises a natural question that, to our knowledge, has not been asked in the CoT-expressivity literature: given two representations of the same combinatorial object related by an explicit, algorithmically realized bijection, do CoT realizations of a derived quantity (here, Strahler number and width) transfer across the bijection, or must they be constructed independently on each side? We answer this by proving identifiability results and associated constructions for both representations (Theorems~\ref{thm:cottreestrahler}, \ref{thm:cotpathstrahler} for Strahler number; Theorems~\ref{thm:cotwidoftree}, \ref{thm:cotwidofpath} for width). The constructions, following \citet{rizvi2024simulating}, involve bilinear maps. While we are interested in proving existence, an interesting open question emerges out of our inquiry: Is CoT realizability closed under bijective changes of representation, and hence, in general, under composition?

    \begin{wrapfigure}[24]{r}{0.45\textwidth}
        \centering
        \resizebox{\linewidth}{!}{
            \begin{tikzpicture}
                \input{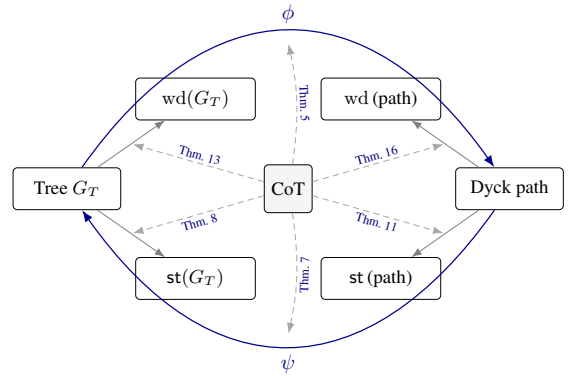}
            \end{tikzpicture}
        }
        \caption{Overview of the paper's constructions and their relationships. The unfolding, $\phi$, maps a tree $G_{T}$ to its Dyck path by recording the DFS traversal used in \autoref{thm:cotdfs}, and the reconstruction, $\psi$, is realized via \autoref{algo:pathtotree} (\autoref{thm:cotpathtotree}). Besides identifying $\str$ and $\wid$ across this bijection, we provide standalone CoT-realization pathways for each directly on both representations: $G_{T}$ (Theorems \ref{thm:cottreestrahler} and \ref{thm:cotwidoftree}, respectively) and paths (Theorems \ref{thm:cotpathstrahler} and \ref{thm:cotwidofpath}).
        }
        \label{fig:contributionflow}
    \end{wrapfigure}
    Below we summarize our main \textbf{contributions}.
    \begin{itemize}[leftmargin=*, itemsep=0pt, wide=0pt]
        \item We give the first CoT realizations of general graph traversal (DFS and Dijkstra) by hard-attention Transformer decoders, each requiring at most two layers and two heads, and show that Dijkstra realization yields a step-count advantage over comparison-based implementations by replacing the search for a minimum-distance vertex with a single constant-time attention operation.
        \item We use these traversal realizations as a shared computational substrate to give explicit, small-depth CoT constructions for two independent notions of branching complexity, the Strahler number and width of a tree, providing a concrete, non-trivial instantiation of the CoT-rank framework after \citet{barcelo2025ehrenfeuchthaussler} on a measure already known to be \textsf{NC\textsuperscript{1}}-complete in the binary case \citep{ganardi2026complexity}, generalized here to arbitrary ordered trees.
        \item Exploiting the classical bijection between trees and Dyck paths, realized algorithmically by our DFS construction, we give independent CoT constructions for both measures on the path representation, and show that the resulting constructions do not transfer readily across the bijection without change in mechanism or layer count.
    \end{itemize}

\noindent
\textbf{Organization.} \autoref{sec:prelim} introduces the necessary preliminaries on graph traversal and branching complexity. \autoref{sec:trans} formalizes the CoT realization framework used throughout. \autoref{sec:results} contains the main results: traversal constructions (DFS and Dijkstra) (\autoref{ssec:foundresults}) and realization of the Strahler number and tree width on both trees and Dyck paths along with the path-reconstruction algorithm (\autoref{ssec:centralresults}). \autoref{sec:conclude} concludes with a comprehensive discussion. Full proofs are deferred to the Appendix.

\noindent\textbf{Related Works.}
    The expressivity of Transformers is well studied, from single-pass characterizations of encoders and decoders \citep{strobl2024formal} to step-indexed accounts of chain-of-thought decoders placing logarithmic and linear reasoning budgets in \textsf{L} and \textsf{NC\textsuperscript{1}} respectively \citep{merrill2024expressive, li2024chain}. In parallel, graph-algorithmic constructions realize DFS, BFS, and Dijkstra via looped Transformers \citep{giannou2023looped} with graph-interacting heads \citep{de2024simulation, sanford2024understanding}, and empirical CoT studies address reachability without materializing traversal \citep{zhu2026emergence}; each lies outside the autoregressive step-counted CoT setting, leaving traversal yet to be explicitly realized as CoT steps. The concrete problems that inhabit the linear-step regime remain few. The Strahler–EH correspondence \citep{dahiya2021onsimple}, together with the \textsf{NC\textsuperscript{1}} membership of Strahler number computation \citep{ganardi2026complexity}, motivates our focus on branching complexity. \autoref{app:relworks} presents a comprehensive discussion.
    
\section{Branching Complexity and Graphs}
\label{sec:prelim}
    \subsection{Graphs and their traversal techniques}
    \label{ssec:prelimgraphs}
        Let $G = (V, E, A)$ be a simple connected directed graph with vertex set $V = \{v_0, \ldots, v_{n-1}\}$, edge set $E \subseteq V \times V$, and adjacency matrix $A: V \times V \to \mathsf{A}$ over an arbitrary domain $\mathsf{A}$, where $a_{ij} = A(v_i, v_j)$ denotes the weight of edge $e_{ij} = (v_i, v_j)$.
    
        \textbf{Dijkstra Algorithm.} Restricting to $\mathsf{A} = \mathbb{R}_{> 0}$ and fixing a source $s = v_p \in V$, the \citeauthor{dijkstra1959note} algorithm computes single-source shortest paths. The instantaneous description of the dynamic at step $t$ is the tuple $\dij^{(t)} = (\dis^{(t)}, cur^{(t)}, vst^{(t)})$, with $\dis^{(t)} \in (\R_{> 0} \cup \{\infty\})^n$ being the tentative distance vector, $cur^{(t)} \in V$ the current vertex, and $vst^{(t)}: V \to \{0,1\}$ the visited indicator. The system is initialized by $cur^{(0)} = v_p$, $\dis^{(0)}[p] = 0$, $\dis^{(0)}[i] = \infty$ for all $v_i \neq v_p$, and $vst^{(0)}(v_p) = 1$ with $vst^{(0)}(v_i) = 0$ otherwise. At each step, given $cur^{(t)} = v_i$, the transition relaxes all outgoing edges by setting $\dis^{(t+1)}[j] = \min(\dis^{(t)}[j], \dis^{(t)}[i] + a_{ij})$ for every $v_j \in V$, selects the next vertex $cur^{(t+1)} = v_{j^\prime}$ where $j^\prime = \argmin_{j}\{\dis^{(t+1)}[j] \mid vst^{(t)}(v_j) = 0\}$, and marks $vst^{(t+1)}(v_{j^\prime}) = 1$ while preserving all prior visits. At termination, $\dis^{(n-1)}[i] = \min_{P \in \mathcal{P}(s,v_i)} w(P)$ for every $v_i \in V$, where $\mathcal{P}(s, v_i)$ denotes the set of directed $(s, v_i)$-paths and $w(P) = \sum_{(v_k, v_l) \in P} a_{kl}$.
    
        Observe that Dijkstra algorithm can be interpreted as a weighted generalization of the BFS traversal. In this context, we redefine $\mathsf{A} = \{1, \infty\}$ to indicate the presence or absence of an edge $e_{ij}$, where a value $1$ denotes that the edge $e_{ij}$ exists and $\infty$ implies that it does not. 
    
        \textbf{Depth-First Traversal.} For DFS traversal \citep{aho1974design} $\Dfs$, we choose $\mathsf{A} = \{0,1\}$. Then, $\nei(v_i) = \{v_j \mid a_{ij} = 1\}$ and let $n = |V|$. The instantaneous description of the dynamic $\Dfs^{(t)}$ is denoted by the tuple $(\dfs^{(t)}, par^{(t)}, vst^{(t)})$, such that $\dfs^{(t)}: V \to V$, $par^{(t)}: V \to V \cup \{\phi\}$ and $vst^{(t)}: V \to \{0,1\}$ returns an spanning tree $T$ from any basic diagraph \citep{gessel1979dfs} $G$ in $O(|V| + |E|)$ moves.  At $t=0$,  given a starting vertex $s =v_p \in V$, $vst^{(0)}(s) = 1$ and $vst^{(0)}(v_i) = 0$ for all $v_i \in V \setminus \{s\}$ and $par^{(0)}(v_i) = \phi$ for all $v_i$. Then the $t$\textsuperscript{th}  dynamic is given by
        \begin{equation}
        \label{eq:dfsdynamic}
            \begin{aligned}
                \dfs^{(t)}(s) &= \begin{cases}
                    v_j \quad \text{ if } v_j \in \nei(\dfs^{(t-1)}(s)) & \text{with } vst^{(t-1)}(v_j) = 0,\\
                    par^{(t-1)}(\dfs^{(t-1)}(s)), & \text{otherwise.}
                \end{cases}
            \end{aligned}
        \end{equation}
        In the former case, \emph{traverse}, $par^{(t)}(v_j) = \dfs^{(t-1)}(s)$ and $vst^{(t)}(v_j) = 1$; all remaining entries of $par^{(t)}$ and $vst^{(t)}$ are inherited from step $t-1$. In the latter case, \emph{backtrack}, the mappings remain unchanged. The algorithm terminates after backtracking from $s$.

        Throughout the remainder, we write $n = |V|$ unless otherwise stated, with $n$ also denoting the length of a lattice path where relevant. Furthermore, whenever the edge set $A$ is omitted in the definition of a graph, we refer to its unweighted variant.

    \subsection{Trees, Paths, and Branching Complexity}
    \label{ssec:prelimtrees}
        If in a graph $G_T = (V, E)$ there is exactly one path between any two distinct vertices $v_i, v_j \in V$, then $G_T$ is a tree, and $|E| = |V| - 1$. For a finite rooted tree $G_T$ \citep{diestel2017graph}, let $\dep(v)$ denote the number of edges on the unique path from the root $r$ to vertex $v$. The height of $G_T$ is defined as $\hgt(G_T):= \max_{v \in V} \dep(v)$, and width as
        \begin{align*}
            \wid(G_T) &:= \max_{k \geq 0}|\{v \in V : \dep(v) = k\}|.
        \end{align*}
        \begin{observation}[Determining $\wid$ while performing BFS]
        \label{obs:dynamicwidbfs}
            To algorithmically determine $\wid(G_T)$, the width of tree $G_T$, we employ BFS as a special version of the Dijkstra algorithm. Let $v_{j}$ and $v_{j^\prime}$ be the vertices such that $vst^{(t)}(v_{j}) = vst^{(t)}(v_{j^\prime})=0$; $vst^{(t+1)}(v_{j}) = 1, vst^{(t+1)}(v_{j^\prime}) = 0$; and $vst^{(t+2)}(v_{j}) = vst^{(t+2)}(v_{j^\prime}) = 1$. Note that, $\dis^{(t+2)}[j^\prime] - \dis^{(t+2)}[j] \in \{0, 1\}$. Clearly, when $\dis^{(t+2)}[j^\prime] = \dis^{(t+1)}[j]$, $\dep(v^\prime) = \dep(v)$; and $\dep(v^\prime) = \dep(v) + 1$, otherwise. To capture the width during $\dij^{(t)}$, we augment it with scalars $twd^{(t)}$ and $mwd^{(t)} \in \Z$ so that the former stores $|\{v_i : vst^{(t)}(i) = 1$ and $\nexists j \ (\dis^{(t)}[j] \ne \infty \land \dis^{(t)}[j] > \dis^{(t)}[i])\}|$ while the latter stores $\max_{t^\prime < t}twd^{(t^\prime)}$. Unless $G_T$ is a null graph, $twd^{(t)}$ and $mwd^{(t)}$ for all $t$ must be at least $1$. Then initialized by $twd^{(0)} = mwd^{(0)} = 1$ indicating the root is at depth $1$, $mwd^{(n-1)} = \wid(G_T)$.
        \end{observation}

        \begin{figure*}
            \centering
            \resizebox{\linewidth}{!}{
                \begin{tikzpicture}
                    \input{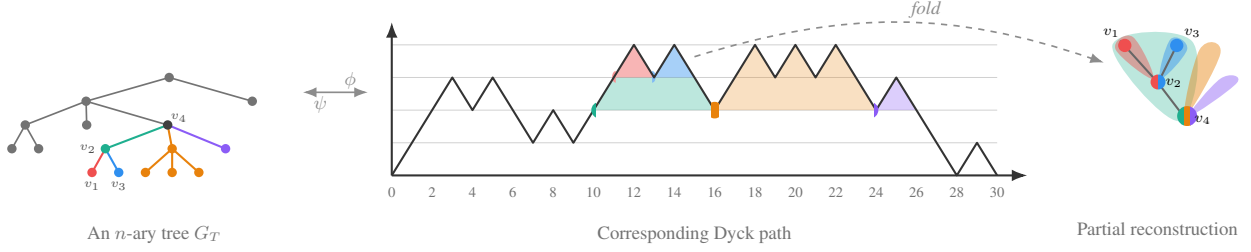}
                \end{tikzpicture}
            }
            \caption{Following the $\Dfs$ dynamic, the map $\phi$ on the tree $G_T$ (left) generates the Dyck path (middle). \iffalse Specifically, a \emph{traverse} step on $G_T$ corresponds to an up step in the Dyck path, while a \emph{backtrack} step corresponds to a down step.\fi We define the reconstruction $\psi$ of the tree from the Dyck word (or path) via \autoref{algo:pathtotree}, with a demonstrative (partial) reconstruction illustrated on the right. The edge pairs $(11,3), (12,4)$ and $(12,4), (13,3)$ are folded into a single tree edge $(v_1, v_2)$. The \colorbox{colC!20}{highlighted region} on the left is formed by folding the region between $(10,2)$ and $(16,2)$. }
            \label{fig:pathfolding}
        \end{figure*}

        Although this is a very natural notion of branching complexity, the \citeauthor{strahler1957quantitative} number $\str$, on the other hand, is calculated recursively. For leaf node $v$, $\str(v) = 0$ and a node with subtrees $\tau_1, \ldots, \tau_k$, let $M = \max_i\,\str(\tau_i)$. Then,
        \[
            \str(G_T) = \begin{cases} 
                M + 1 & \text{if } c=|\{i : \str(\tau_i) = M\}| \geq 2 \\
                M & \text{otherwise.} 
            \end{cases}
        \]
        For binary trees, this reduces to $\str(G_T) = \max(\min(\str(\tau_1), \str(\tau_2))+1, \max(\str(\tau_1), \str(\tau_2)))$, which, interestingly, coincides with the expression of the Ehrenfeucht-Haussler rank for binary decision trees.
        \begin{observation}[Determining $M$ and $c$ whiler performing $\Dfs$]
        \label{obs:dynamicstrdfs}
            We reuse the dynamic formulation \eqref{eq:dfsdynamic}, which is evaluated in a bottom-up manner. Specifically, when the dynamic process is in a \emph{traverse} step, we assign $t_{\dfs^{(t)}(s)}$, the tree rooted at the vertex $\dfs^{(t)}(s)$, the value $0$. Note that in this situation, $\tau_{\dfs^{(t)}(s)}$ consists of a single vertex. During a \emph{backtrack} step, let $v = \dfs^{(t)}(s)$ and $u = \dfs^{(t+1)}(s)$. At this point, we record $\str(\tau_u)$ for the tree rooted at the vertex $u$ that has been explored up to step $t$. We let $M_u^{(t)}$ and $c_u^{(t)}$ denote, respectively, the maximum Strahler number observed among the already explored children of $u$ other than $v$\footnote{or one more than it, in a boundary case when $D=0$ is observed a priori.}, and the number of children attaining this maximum value. Analogously, $M_v^{(t)}$ and $c_v^{(t)}$ denote the maximum Strahler number observed among the children of $v$ and the corresponding multiplicity of this maximum.Then $\str(\tau_v) = M_v^{(t)} + \ind[c_v^{(t)}\ge2]$. Let $D = \str(v) - M_u^{(t)}$.  
            \begin{equation}
            \label{eq:dynamicstrdfs}
                \begin{aligned}
                    D & \begin{cases}
                        < 0 & M_{v}^{(t+1)} = M_u^{(t)}, c_v^{(t+1)} = c_u^{(t)}\\
                        = 0 & \begin{cases}
                                M_{v}^{(t+1)} &= M_u^{(t)} + \ind[c_u^{(t)} \ge 1]\\
                                c_v^{(t+1)} &= \ind[c_u^{(t)} +1 < 2]
                            \end{cases}\\
                        > 0 & M_{v}^{(t+1)} = \str(\tau_v), c_v^{(t+1)} = 1,
                    \end{cases}
                \end{aligned}
            \end{equation}
            which reduces to 
            \begin{equation*}
                \begin{aligned}
                    M_{v}^{(t+1)} &= \max(\str(\tau_v), M_u^{(t)}) + \ind[D=0] \ind[c_u^{(t)} > 0]\\
                    c_v^{(t+1)} &= \ind[D<0]c_u^{(t)} + \ind[D=0]\ind[c_u^{(t)}=0] + \ind[D>0]1.
                \end{aligned}
            \end{equation*}
        \end{observation}
        
        A lattice path of length $n$ with step set $\mathcal{S} \subseteq \mathbb{Z}$ is a sequence $(0,s_0), (1,s_1), \ldots, (n,s_n)$ with $s_0=0$ and $s_i - s_{i-1} \in \mathcal{S}$ for all $i \geq 1$; it is an \emph{excursion} \cite{flajolet2009analytic} if $s_i \geq 0$ throughout and additionally $s_n=0$. An excursion with step-set $\mathcal{S}=\{+1,-1\}$ is famously known as a Dyck path. By writing $\UU$ for $+1$ and $\DD$ for $-1$, Dyck paths are presented as Dyck words over $\{\UU,\DD\}$. The bijection between trees and these paths is well known \cite{goldman1992lattice}. Given $G_T$, we run a DFS starting from root $r$ and record a $\UU$ each time the first case of the dynamic $\dfs$ occurs, and a $\DD$, otherwise. Since every one of the $n-1$ edges is traversed once in each direction, the resulting word $\phi(G_T)$ has length $2(n-1)$. Going the other way, given an excursion  $w$ of length $2m$, we build the tree $\psi(w)$: a $\UU$ creates a new node attached as the next ordered child of the current node and moves the pointer down to it, while a $\DD$ returns to the parent (see \autoref{fig:pathfolding}). Therefore,

        \begin{lemma}
        \label{lem:bijectionpathtree}
            There exists maps $\phi$ and $\psi$, mutually inverse bijections between the set $\mathcal{T}_n$ of all finite ordered trees and the set $\mathcal{D}_{n-1}$ of Dyck words of length $2(n-1)$.
        \end{lemma}

        The bijection $\phi$ reduces the combinatorial structure of $G_T$ to the geometry of $\phi(G_T)$.  We make this precise for both $\hgt$ and $\wid$.  Throughout, for $G_T \in \mathcal{T}_n$ we let $(s_0, s_1, \ldots, s_{2(n-1)})$ denote the ordinate sequence of the lattice path induced in $\phi(G_T)$: concretely, $s_0 = 0$ and $s_i - s_{i-1} = +1$ or $-1$ according to the $i$-th letter of $\phi(G_T)$ being $\UU$ or $\DD$, so that $s_i = \sum_{j=1}^{i} q_j$ where $q_j \in \{+1, -1\}$ is the signed value of the $j$-th step.

        \begin{proposition}
        \label{prop:heightpathtree}
            For every $G_T \in \mathcal{T}_n$, $\hgt(G_T) = \max_{0 \le i \le 2(n-1)} s_i$.
        \end{proposition}

        \begin{proposition}
        \label{prop:widthpathtree}
            For every $G_T \in \mathcal{T}_n$ with $n \ge 2$, $ \wid(G_T) = \max_{0 < l \le \hgt(G_T)} |\{i : s_i = l, s_i - s_{i-1} = 1\}|$.
        \end{proposition}

        \begin{remark}[Regularity and Relation between the Metric Measure Spaces]
            The existence and identifiability results above become even more important given the scarcity of similar evidence in a geometric setting, let alone their algorithmic realizability. Plane trees with $n$ edges induce a deterministic bijection with Dyck paths of length $2n$ via the contour walk \citep{viennot2002strahler} and with {\L}ukasiewicz paths via the depth-first walk \citep{Gall2005tree}. However, the regularity of the same is not trivially observed. One can only ensure at the continuum for a continuous excursion $g : [0,1] \rightarrow \mathbb{R}_{+}$ into the metric measure space of trees $\mathcal{T}_{g}$\footnote{endowed with the \textit{pseudometric} $d_{g}(s,t) \coloneqq g(s) + g(t) - 2 \min_{[s \wedge t, s \vee t]}g$, after quotienting by $\{d_{g} = 0\}$ and \textit{measure} $\mu_{g} = (p_{g})_{*}\mu$, where $\mu$ is the Lebesgue measure.} the $2$-Lipschitz-regularity of the forawrd map (in terms of the Gromov-Hausdorff distance) (\citet{evans2008probability}, \citet{Gall2005tree}, Lemma 2.2.3). Then, the inverse embedding (tree $\rightarrow$ path) also fails to be Lipschitz, and thus admits no isometric isomorphism at the continuum level. A deterministic isomorphism of lattices (and hence a graph isometry of the associated metric spaces) can only be realized under triangulation (both endowed with the Tamari partial order) \citep{bernardi2009iso}. While we leave checking the regularity of $(\psi, \phi)$ as future work, it can be noted as a first attempt at scrutinizing their CoT realizability.
        \end{remark}
        
        To construct $\psi(w)$ from $w$, we introduce \autoref{algo:pathtotree} that returns the adjacency matrix of the tree on receiving a Dyck path as input (\autoref{app:maintextprelim} provides the proof of correctness).
        \begin{algorithm}
            \SetAlgoLined\SetAlgoNoLine\SetNlSty{}{}{}\LinesNumbered\RestyleAlgo{ruled}\DontPrintSemicolon\SetAlgoSkip{}\SetKwComment{tcp}{\textcolor{blue}{$\vartriangleright$ }}{}\SetCommentSty{texttt}
            \SetKwIF{If}{ElseIf}{Else}{if}{}{else if}{else}{}
            \SetKwFor{For}{for}{}{}
            \KwIn{A Dyck word $w$ corresponding to lattice path $\{(t,s_t)\}_{t=0}^{2(m-1)}$.}
            Let $A_w$ be the adjacency matrix of $\psi(w)$ initialized to $0$.\;
            $nl \gets 0, i \gets 0$.\;
            \For{$c \in w$}{
            \eIf{$c$ is $\UU$ \tcp*[f]{i.e. $s_{t+1} - s_{t} = 1$.}}{
                $nl \gets nl + 1$.\;
                    $A_w[i, nl] \gets 1, A_w[nl, i] \gets 1$.\;
                    $i \gets nl$.\;
                }{
                    Search $j < i$ s.t. $A_w[j, i] = 1$.\;
                    $i \gets j$.\;
                }
            }
            \Return $A_w$.
            \caption{ Reconstruction: Path $\rightarrow$ Tree ($\psi$) }
            \label{algo:pathtotree}
        \end{algorithm}

\section{Transformer Decoder and CoT}
\label{sec:trans}
    The CoT steps formalize, following \citet{merrill2024expressive, li2024chain}, the autoregressive nature of the Transformer decoder. We first describe the \emph{unique hard}-attention layer and then the full decoder architecture. Our framework largely follows \citet{barcelo2025ehrenfeuchthaussler}, with two necessary modifications that we identify in our exposition.

    \textbf{Unique Hard-Attention.} Given an input sequence $(x_0, \ldots, x_{m-1}) \in (\mathbb{R}^d)^m$, let $Q, K \in \mathbb{R}^{d^\prime \times d}$ denote the query and key projection matrices, respectively. The attention score between the position $i$ and the query position $m-1$ is defined as
    $
        a_{i,m-1} = \langle Kx_i,\, Qx_{m-1} \rangle.
    $
    With respect to the $m$-th query $Qx_{m-1}$, the scalar $a_{i,{m-1}}$ quantifies the affinity of the key $Kx_i$. Let $j_0 < j_1 < \cdots < j_p \in \{0, \ldots, {m-1}\}$ denote the indices, arranged in ascending order, that jointly attain the maximum attention score. Unique hard attention selects the smallest such index $j_0$. Departing from \citet{barcelo2025ehrenfeuchthaussler}, we retain the standard value matrix\footnote{The symbol $V$ is overloaded to denote both the vertex set of a graph and the value projection in attention; the intended meaning will be clear from context.} $V \in \mathbb{R}^{d^{\prime\prime} \times d}$, so that the attention produces $Vx_{j_0}$.

    \textbf{Multi-Head Attention.} A multi-head extension uses $H$ attention heads, where head $h$ ($h\in [H]$) is parameterized by $(Q_{h}, K_{h}, V_{h})$. Let $j_0^{(h)}$ be the index selected by head $h$; the combined output is
    \[
        \alpha^{\prime\prime} = W_O \cdot \left[\left(V_{1}x_{j_0^{(1)}}\right)^\top \Big\|\  \cdots \Big\|\  \left(V_{H}x_{j_0^{(H)}}\right)^\top\right]^\top,
    \]
    where $\parallel$ denotes vector concatenation, $W_O$ is the output projection matrix, and $\alpha^{\prime\prime} \in \R^d$.

    \textbf{Single-Layer Decoder.} A position-wise feed-forward network is applied to the residual $\beta + x_{m-1}$ via two successive linear projections interleaved with a ReLU nonlinearity, where $\operatorname{ReLU}(r) := \max(0, r)$ for $r \in \mathbb{R}$, applied coordinate-wise for $r \in \mathbb{R}^d$. The single-layer decoder thus computes
    \[
        L(x_0, \ldots, x_{m-1}) = W_2\,\operatorname{ReLU}\!\left(W_1(\beta + x_{m-1})\right),
    \]
    for weight matrices $W_1 \in \mathbb{R}^{d^{\prime\prime\prime} \times d}$ and $W_2 \in \mathbb{R}^{d \times d^{\prime\prime\prime}}$. In our constructions, we will often use $\alpha^{\prime}$ to denote $Vx$. Each input vector $x_i$ is constructed from a word embedding $\operatorname{WE}(\sigma_i)$ and a positional encoding $\operatorname{pos}(i)$ via a composition $x_i = g(\operatorname{WE}(\sigma_i), \operatorname{pos}(i))$. Unlike \citet{barcelo2025ehrenfeuchthaussler}, we adopt the simplification $g(\operatorname{WE}(\sigma_i), \operatorname{pos}(i)) = \operatorname{WE}(\sigma_i) =: x_i$. A single-layer Transformer decoder, upon receiving input $(x_0, \ldots, x_{m-1}) \in (\mathbb{R}^d)^m$ and an initial \textsf{eos} token $y_0 \in \mathbb{R}^d$, generates output sequence $\{y_t \in \mathbb{R}^d\}_{t=1}^{\infty}$ via the recurrence
    \[
        y_t = L(x_0, \ldots, x_{m-1}, y_0, \ldots, y_{t-1}), \quad t \geq 1.
    \]
    That is, at each step $t$, the decoder appends its most recent output $y_{t-1}$ to the running context and applies $L$ to the resulting extended sequence.

    \textbf{Multi-Layer Extension.} A depth-$\ell$ Transformer decoder is obtained by composing $\ell$ such single-layer decoders $L^{(1)}, \ldots, L^{(\ell)}$, each with its own independent parameter set $\{Q^{(k)}_{h}, K^{(k)}_{h}, V^{(k)}_{h}, W_O^{(k)}, W_1^{(k)}, W_2^{(k)}\}_{h=1}^{H_k}$ for layer $k \in [\ell]$. Given the input context $(x_0, \ldots, x_{m-1})$, the layers process the sequence in depth-first order: at generation step $t$, the intermediate representations are computed as 
    \[
        z_m^{(0)} \coloneqq x_{m-1}, z_m^{(k)} \coloneqq L^{(k)}\!\left(z_1^{(k-1)}, \ldots, z_m^{(k-1)}\right),
    \]
    for $k \in [\ell]$ and the final output at step $t$ is $y_t := z_{m_t}^{(\ell)}$, where $m_t = m + t$ denotes the current sequence length. The newly produced token $y_t$ is appended to the context, after which the full $\ell$-layer computation is repeated for step $t+1$.
    
    \begin{definition}[CoT Realization]
    \label{def:algorelationbycot}
        Let $\mathcal{A}$ be an algorithm that, on input, $I$ say, produces a sequence of states $\{\mathsf{s}_t\}_{t=1}^{T}$, where $T$ denotes the number of iterations until termination. Let $\pi : \mathbb{R}^d \to \mathbb{R}^{d_\mathsf{s}}$ denote the projection onto the designated output coordinates. We say that a decoder \emph{realizes} $\mathcal{A}$ via CoT if there exists a parameter assignment such that, for every valid input and every $t \in [T]$, $\pi(y_t) = \operatorname{enc}(\mathsf{s}_t),$ where $\operatorname{enc}$ is a fixed encoding of algorithmic states into $\R^{d_\mathsf{s}}$. The generation runs for exactly $T$ steps; the step count $T$ is determined a priori, and hence termination follows immediately.
    \end{definition}
    
\section{Main Results}
\label{sec:results}
    This section is organized as follows. First, we present the construction of the graph traversal techniques (see section~\ref{ssec:foundresults}). Section~\ref{ssec:centralresults} then introduces the theorems necessary to establish the main results, as summarized in \autoref{fig:contributionflow}. The bilinear maps employed in our constructions, following \citet{rizvi2024simulating}, are in general not unique, as can be seen by comparing \autoref{thm:cottreestrahler} and \ref{thm:cotpathstrahler}. However, substituting these bilinear maps with FFN blocks from a Transformer layer would, in several cases where merging independent operations is infeasible, alter the required number of layers.

    Bilinear maps have been employed in the study of Transformers in both theoretical analyses \citep{rizvi2024simulating} and empirical investigations, particularly within the contexts of multi-modal and multi-view learning \citep{li2017factorized, gao2016compact}. In our work, we adopt analogous bilinear projection mechanisms as fundamental components of our constructions. We apply linear, affine, and bilinear maps to independent resultant blocks using a single transformation. In \autoref{app:prelim}, we show that all of these can be written as a single bilinear map. Also, any scalar–vector operation is applied elementwise to the vector.

    \subsection{Results on Graph Traversal}
    \label{ssec:foundresults}
        \begin{restate}{thm:cotdfs}
            A two-layer, two-head Transformer decoder can simulate the depth-first traversal $\Dfs$ on any simple directed graph $G = (V, E, A)$ in $O(|V| + |E|)$ chain-of-thought steps.
        \end{restate}
        \begin{proof}[Proof Sketch]
            The construction realizes the dynamic $\dfs^{(t)}$ as in \eqref{eq:dfsdynamic} and proceeds in two layers. The first layer uses a trivial attention mechanism and determines, via its feed-forward block, whether the current vertex has an unvisited neighbor — distinguishing a \emph{traverse} from a \emph{backtrack}. The second layer executes the chosen operation through two mutually exclusive attention heads, one resolving the next unvisited neighbor and the other retrieving the parent vertex (see \autoref{app:cotdfs}).
        \end{proof}

        When the underlying graph is a tree $G_T$, the number of required CoT steps is $2(n-1)$. The DFS realization plays a central role in two distinct respects: \textit{i}. The mapping $\phi$ coincides exactly with the execution of the procedure $\Dfs$ on the tree $G_T$. To obtain the associated Dyck path (or Dyck word), it suffices to extend the token embedding by one additional component that records $+1$ (or $\UU$) and $-1$ (or $\DD$) exclusively, in accordance with the \emph{traverse} and \emph{backtrack} operations, in the final layer. \textit{ii}. The same DFS realization is also reused in the implementation of \eqref{eq:dynamicstrdfs}, thereby yielding the Strahler number of a tree as stated in \autoref{thm:cottreestrahler}.

        \begin{restate}{thm:cotdijkstra}
            A two-layer, single-head Transformer decoder can simulate the Dijkstra algorithm $\dij$ on any simple connected directed graph $G = (V, E, A)$ with $\mathsf{A} = \R_{> 0}$ in exactly $|V|-1$ chain-of-thought steps.
        \end{restate}
        \begin{proof}[Proof Sketch]
            The proof proceeds by implementing the dynamics of $\dij^{(t)}$ as described in \autoref{ssec:prelimgraphs}. In the first layer, the attention collects the edge weights by identifying the current vertex from input tokens. Since $\mathsf{A}\in \R_{>0}$, its FFN block updates the distance vector $\dis$ by applying the relaxation operation to all neighboring vertices. The second layer then uses its attention mechanism to select the unvisited vertex with the smallest $\dis$, updates its representation, and marks the selected vertex as visited (see \autoref{app:cotdij}).
        \end{proof}

        The classical implementation of Dijkstra algorithm requires $O(n^2)$ time, whereas efficient variant runs in $O(|E| + n \log n)$. In contrast, our implementation requires at most $n-1$ CoT steps. This yields a substantial computational advantage: the Transformer avoids explicitly searching for the minimum-distance unvisited vertex by using constant-time attention. We use this theorem, along with \autoref{obs:dynamicwidbfs}, to compute $\wid$ in \autoref{thm:cotwidoftree}.

    \subsection{CoT Realization of Complexity Measures}
    \label{ssec:centralresults}
        After proving \autoref{algo:pathtotree} in \autoref{thm:cotpathtotree} (see \autoref{app:cotpathtotree}), \autoref{thm:cottreestrahler} and \ref{thm:cotpathstrahler} present the proofs for computing the Strahler number treating tree and Dyck path as inputs, respectively, under CoT (\autoref{app:cotstr}). \autoref{thm:cotwidoftree} and \ref{thm:cotwidofpath} provide the proofs for computing the measure $\wid$ (\autoref{app:cotwid}).

        \begin{restate}{thm:cotpathtotree}
            A single-layer, two-head Transformer decoder can simulate \autoref{algo:pathtotree} in the $2(m-1)$ CoT-steps\footnote{In this derivation alone, we write the length of the Dyck path for a tree of $m$ nodes as $2(m - 1)$, to make the distinction explicit; elsewhere we retain $n$, as mentioned earlier, for the number of nodes or the Dyck path length.}.
        \end{restate}
        \begin{proof}[Proof Sketch]
    		Each step executes one iteration of the loop on line 3, reading the symbol $w_t$ and updating $(i, nl, A_w)$ in place. A single attention head reads $w_t$ by matching the cursor $pos$ against the input tokens, returning its signed increment $q_t \in \{+1, -1\}$; a second head, querying on the current vertex, retrieves from that vertex's creation token the parent needed on a $\DD$ step (line 9). The feed-forward block then branches on the sign of $q_t$ entirely within a single $\relu$, via $\relu(v + q_t - 1) = \ind[q_t = +1]\,v$: on a $\UU$ it increments $nl$, writes the symmetric edge $\{i, nl\}$ as the one quadratic feature of $W_1$, and advances $i$ to the new vertex (lines 5--6); on a $\DD$ it moves $i$ to the retrieved parent. The cursor advances each step, and the chain halts when the path exhausts. Induction on $t$ shows $\pi_A(y_t)$ equals the adjacency matrix of $\psi(w)$ after reading $w_0 \cdots w_{t-1}$.
    	\end{proof}
        \begin{restate}{thm:cottreestrahler}
            A four-layer, two-head Transformer decoder can compute $\str(G_T)$, the Strahler number of a tree $G_T$ during the simulation of $\Dfs$ in exactly $2n - 1$ CoT steps, where $n$ denotes the number of vertices in $G_T$.
        \end{restate}
        \begin{proof}[Proof Sketch]
            The decoder of \autoref{thm:cotdfs} traverses $G_T$. We extend its embedding alone, giving each token a pair $(\pi_M, \pi_c)$: the running maximum of the Strahler numbers of the current vertex's finished children, and the count attaining it. A timestamp $\pi_{flg_4}$ turns the backtrack head's leftmost selection into a rightmost one, so it resolves the parent as last left and reads the accumulator it then carried. The feed-forward blocks of two further layers fold the finished child into this accumulator. The comparison against the parent's maximum is decided through indicators of $D = \str(v) - M_u$ and of the count, laid down as differences of rectified blocks. A carry lifts the maximum when the count saturates, keeping the fold bilinear. The accumulator is never finalized in place; the final \emph{backtrack} folds the source under the same invariant, giving $\pi_M(y_{2n-1}) = \str(G_T)$ in $2n-1$ steps.
        \end{proof}
        \begin{restate}{thm:cotpathstrahler}
            A four-layer, single-head Transformer decoder can compute $\str(\psi(w))$, the Strahler number of a tree corresponding to the Dyck word $w$ in exactly $n$ CoT steps, where $n=|w|$.
        \end{restate}
        The proof of the above theorem builds on the same techniques as before, in particular the implementation of \eqref{eq:dynamicstrdfs} from \autoref{thm:cottreestrahler}. For this reason, it is only outlined in \autoref{app:cotstr}, where we provide the additional arguments needed to handle paths as input. Note that while \autoref{thm:cottreestrahler} directly computes the $\str$ number of the tree $G_T$ in its final CoT step, the above theorem provides the values $M$ and $c$ in the designated positions as specified by $\pi_M(y_{n})$ and $\pi_c(y_{n})$ so that $\str(\psi(w)) = \pi_M(y_{n}) + \ind[\pi_c(y_{n}) \ge 2]$.
        \begin{restate}{thm:cotwidoftree}
            A three-layer single-head Transformer decoder can find the width of a tree $\wid(G_T)$ during the simulation of $\dij$ in exactly $n-1$ steps, where $n$ denotes the number of vertices in $G_T$.
        \end{restate}
        \begin{proof}[Proof Sketch]
            We augment the embeddings of the tokens of \autoref{thm:cotdijkstra}  by appending three scalar coordinates $(\pi_{dif}, \pi_{twd}, \pi_{mwd})$, initialized to $(0, 1, 1)$ at the root, tracking a BFS level-change indicator, the running node count at the current level, and the global maximum width. Layers 1-2 replicate the Dijkstra construction, with the sole modification that $W_2^{(2)}$ additionally writes $\pi_{dif}(y_t^{(2)}) \leftarrow \sum_{i} \pi_{buf}(y_t^{(1)\prime})[i] - \pi_{crd}(y_t)$, the difference between the BFS depths of the current and preceding decoded nodes. Since $G_T$ carries weights in $\{1,\infty\}$ and nodes are decoded in non-decreasing BFS-depth order, \autoref{obs:dynamicwidbfs} guarantees $\pi_{dif} \in {0,1}$, with $\pi_{dif} = 1$ signalling a transition to a strictly deeper level. Layer 3 employs a trivial attention sub-layer followed by an FFN that, when $\pi_{dif} = 0$, increments $\pi_{twd}$ and updates $\pi_{mwd}$ with $\max(\pi_{mwd}, \pi_{twd})$, and when $\pi_{dif} = 1$, commits the current count to $\pi_{mwd}$ and resets $\pi_{twd}$ to $1$. After $n-1$ decoding steps, each BFS level's count has been committed to $\pi_{mwd}$, and thus $\pi_{mwd}(y_{n-1}) = \wid(G_T)$.
        \end{proof}
        \begin{restate}{thm:cotwidofpath}
            A two-layer single-head Transformer decoder can find $\wid(\psi(w))$ for a Dyck word $w$ in exactly $n$ CoT steps, where $|w| = n$.
        \end{restate}
        \begin{proof}[Proof Sketch]
            The embedding dimension $d = \tfrac{3n}{2}+4$ is partitioned into a one-hot position block, a step scalar $q_t$, a cumulative height scalar $ht$, a count block $p \in \R^{\frac{n}{2}+1}$ where $p_i$ tallies up-steps reaching height $i$, and a width scalar $wd$. In the first Layer, attention retrieves the current step $q_\tau$ via a shift matrix, and the FFN applies a discrete pulse function (\autoref{lem:addonetospecificheight}) to increment exactly the $p_i$ entry indexed by the current height $ht$. In the final layer, the FFN performs a global max-aggregation (\autoref{lem:addwidplus1}) updating $wd$ to $\max_i p_i$ while zeroing the step scalar. Correctness follows since each $p_i$ counts nodes at depth $i$ in $\psi(w)$, so the final $wd$ equals $\wid(\psi(w))$ as claimed in \autoref{prop:widthpathtree}.
        \end{proof}
        A natural question arises regarding the optimal use of layers: given a tree $G_T$, is there an advantage to directly computing its width $\wid(G_T)$ using a three-layer Transformer (as in \autoref{thm:cotwidoftree}), rather than breaking the process down? For instance, an alternative composite setup might combine a two-layer, two-head Transformer to implement the map $\phi$ (in \autoref{thm:cotdfs}) with a two-layer, single-head Transformer to compute width on the corresponding Dyck path (in \autoref{thm:cotwidofpath}). If CoT is closed under composition and layer counts add linearly, the first approach is strictly more economical in layer count. Moreover, directly employing the three-layer Transformer offers another benefit: it reduces the required CoT steps significantly compared to the composite method.
        
        \begin{remark}[Closure under Composition]
            Given the fact that CoT-Transformers fail to solve Compositional Reasoning Questions (CRQ) unless CoT tokens are allowed to grow accommodating compositional depth \citep{yehudai2025compositional}, it is clear that the composition does not replicate itself. However, the same constructive realizability witness result for composition over tree-structured tasks hints towards a linear complexity class. Notably, our construction preserves linear token growth. On the other hand, since ensuring CoT-learnability essentially boils down to the finiteness of the VC-dimension of the base classes $\mathcal{F}_{i}$ \citep{joshi2025theory}, it would be sufficient for them to satisfy a VC upper bound showing linear aggregation in the spirit of \citet{noga2021vc}, Proposition 2: $\textrm{VC}(G(\mathcal{F}_1,\ldots,\mathcal{F}_k)) \lesssim \textrm{VC}(G) + \sum_{i=1}^{k} \textrm{VC}(\mathcal{F}_i)$, to prove closure under composition along those lines, where $G$ specifies composition, $k > 0$.  
        \end{remark}

\section{Conclusion}
\label{sec:conclude}
    To our knowledge, the CoT realizations developed here are the first of their kind in the literature. This immediately invites the question of optimality: identifying constructions that minimize embedding dimension, layer depth, and number of attention heads constitutes a natural and important direction for future work. A second open question concerns the closure properties of CoT under the tree–path bijection. The CoT realizations of the Strahler number for path and tree inputs exhibit remarkably little structural synergy — beyond their shared adherence to \eqref{eq:dynamicstrdfs} — despite the bijective correspondence between the two input representations via $\phi$ and $\psi$. A further observation concerns the number of heads required in \autoref{thm:cottreestrahler} and \autoref{thm:cotpathstrahler}. Intuitively, a single head can scarcely distinguish between a \emph{traverse} and a \emph{backtrack} along the same edge, whereas the linear progression of $\UU$ and $\DD$ in a Dyck path makes this distinction straightforward, so one head suffices. This also explains the two heads in \autoref{thm:cotpathtotree}: reconstructing the tree from the path re-encodes each $\DD$ with its corresponding $\UU$, folding the information of both the \emph{traverse} and \emph{backtrack} steps back onto a single edge {--} and it is this folding that necessitates the second head.

\section*{Limitations}
\label{sec:limitation}
    The constructions presented herein are not necessarily unique, and we acknowledge that alternative formulations satisfying the same theoretical requirements may well exist. As one illustration, \autoref{thm:cottreestrahler} introduces bilinear gating at the final layer to suppress redundant operations on the $M$ and $c$ blocks during a \emph{traverse} step; \autoref{lem:mMc}, however, demonstrates that bilinearity can be entirely circumvented to achieve the same effect in the context of \autoref{thm:cotpathstrahler}. As one limitation, several of our constructions invoke the bilinear operation. Following \citet{rizvi2024simulating}, however, this usage stays within the established conventions of this line of work rather than constituting a departure from them. In this regard, unlike \citet{merrill2024expressive, qiu2025ask}, our constructions forgo supplementary computational primitives such as layer normalization. Rather than a shortcoming, this economy proved crucial for constructing CoT Transformers that reach \textsf{NC\textsuperscript{1}}, showing that such augmentations are not a necessary architectural ingredient for CoT constructions. Additionally, we note that \autoref{thm:cotpathstrahler} concludes with the final values $M$ and $c$, the closing operation $M + \ind [c \ge 2]$ that fixes the Strahler number sits marginally outside the CoT iteration. Furthermore, there is also a classical correspondence with plane trees and {\L}ukasiewicz paths. We have not studied CoT realizations for these representations, which may yield additional architectural insights in the broader theory of computation.

\bibliographystyle{bibstyle} 
\bibliography{references}  
\appendix
\section{Related Works}
\label{app:relworks}
    A substantial body of work delineates what Transformers can compute in a single forward pass, for encoders \citep{hahn2020theoretical, perez2021attention, weiss2021thinking, hao2022formal, chiang2023tighter, sanford2023representational, barcelo2024logical} and decoders alike \citep{perez2021attention, merrill2024expressive, peng2024on, barcelo2025ehrenfeuchthaussler}, surveyed by \citet{strobl2024formal}. The recurring conclusion is a ceiling: unique and generalized hard-attention encoders recognize only languages in \textsf{AC\textsuperscript{0}} \citep{hao2022formal}, which already excludes the Dyck languages. Later work sharpens the picture to exact automata-theoretic characterizations \citep{rizvi2024simulating} and depth hierarchies \citep{merrill2026little}. As such, it is fair to say that we work in the most restrictive of these regimes (unique hard attention, no layer normalization) and the expressive power we exhibit is attributable to CoT alone.
    
    Building on \citeauthor{perez2019turing}'s result that hard-attention decoders with unboundedly many steps compute any decidable language, \citeauthor{merrill2024expressive} refine this into a step-counted hierarchy (log steps $\rightarrow$ \textsf{L}, linear steps $\rightarrow$ \textsf{NC\textsuperscript{1}}); \citet{li2024chain} obtain constant-depth constructions for \textsf{P}/\textsf{poly}, \citet{nowak2024representational} characterize CoT-augmented language models, and related gains are known for pause and filler tokens \citep{london2026pause} and continuous thought \citep{zhu2026emergence}. While these characterize classes, concrete problems instantiating them remain scarce. In this context, we reiterate \citeauthor{barcelo2025ehrenfeuchthaussler}'s contribution in showing that the EH rank is exactly the minimum number of CoT steps for a single-layer hard-attention decoder. Since the EH rank of a binary decision tree obeys the same recurrence as its Strahler number \citep{dahiya2021onsimple}, and computing the latter from a term representation is \textsf{NC\textsuperscript{1}}-complete \citep{ganardi2026complexity}, branching complexity probes precisely the linear-step regime rather than being an arbitrary target.
    
    A parallel thread constructs Transformers that execute procedures: looped architectures as programmable computers \citep{giannou2023looped}, parsing during masked prediction \citep{zhao2023transformers}, universal simulation of attention \citep{dutta2025existence}, and simulation of weighted automata over sequences and trees by bilinear maps \citep{rizvi2024simulating}, which we adopt as a primitive. For graph algorithms specifically, \citet{de2024simulation} simulate DFS, BFS, Dijkstra, and Kosaraju by \textit{looped} Transformers with graph-interacting heads and graph-size-independent parameter count; \citet{sanford2024understanding} chart architecture tradeoffs for graph reasoning. However, all of the above evade the CoT decoder model. Looping is not autoregressive generation, and neither line counts CoT steps. Within CoT, \citet{zhu2026emergence} study reachability (to identify which nodes in a graph can be reached from a specified source node) empirically without realizing the traversal, and \citet{barcelo2025ehrenfeuchthaussler} invoke exhaustive tree traversal as a proof device without constructing it or handling weights. We close this gap with explicit decoder constructions, making traversal an object of study rather than scaffolding. Notably, \citeauthor{ganardi2026complexity}'s own \textsf{NC\textsuperscript{1}} upper bound proceeds by a heavy-subtree-first DFS, so a CoT realization of DFS is the natural bridge to their result.

    On the tree–path correspondence end, while it admits several classical realizations \citep{viennot2002strahler, Gall2005tree}, its regularity is delicate even at the continuum \citep{evans2008probability}, and none has previously been examined for CoT realizability.

\section{Missing Proofs of Lemmas and Propositions}
\label{app:maintextprelim}
    \begin{proof}[Proof of \autoref{lem:bijectionpathtree}]
        Suppose $\phi(T)$ is not a Dyck word for some $T \in \mathcal{T}_n$. Then for some edge $((i, s_i), (i+1, s_{i+1}))$ on $\phi(T)$ labeled with $\DD$ $s_{i+1} < 0$ or $s_{2(n-1)} \ne 0$. If the former is the case, there must be more $-1$ steps than $+1$ steps. If this is the case, some edges in $T$ are backtracked more than they are traversed, which is absurd. On the other hand, had $s_{2(n-1)} \ne 0$, specifically $s_{2(n-1)} > 0$, then we would not have backtracked some edges. This would, by definition, violate the DFS traversal scheme.
        
        We verify $\psi(\phi(T))=T$ by structural induction on $n$. The case $n=1$ is clear: $\phi(T)=\varepsilon$ and $\psi(\varepsilon)$ is a lone root $r$. If the root has ordered children subtending subtrees $t_1,\ldots,t_k$, then $\phi(t) = \UU \phi(t_1) \DD \cdots \UU \phi(t_k) \DD$. Reading this with $\psi$, the $j$-th block $\UU \phi(t_j) \DD$ creates a new child of the root and reconstructs $t_j$ beneath it by the inductive hypothesis. The identity $\phi(\psi(w))=w$ follows by the same argument on $|w|$.
    \end{proof}

    \begin{proof}[Proof of \autoref{prop:heightpathtree}]
        Since $G_T$ and $\phi(G_T)$ are both finite, both $\hgt(G_T)$ and $\max_{0 \le i \le 2(n-1)} s_i$ are finite. Let $u_i \in V(G_T)$ denote the node at which the DFS pointer rests after the processing step $i$, so $u_0 = r$.  We claim
        \begin{equation}
        \label{eq:depth-height}
            \begin{aligned}
                & s_i = \dep(u_i) \quad (0 \le i \le 2(n-1)), \quad \text{ and } \quad  s_i - s_{i-1} = 1 \quad (1 \le i \le 2(n-1)).
            \end{aligned}
        \end{equation}
        The base case $s_0 = 0 = \dep(r)$ is immediate.  For the inductive step, a $\UU$ at position $i$ moves the pointer to a child of $u_{i-1}$, giving $\dep(u_i) = \dep(u_{i-1}) + 1 = s_{i-1} + 1 = s_i$.  A $\DD$ at position $i$ moves it to the parent of $u_{i-1}$, giving $\dep(u_i) = \dep(u_{i-1}) - 1 = s_{i-1} - 1 = s_i$.  This establishes \eqref{eq:depth-height}.
         
        Since DFS visits every node of $G_T$, the map $i \mapsto u_i$ is surjective onto $V(G_T)$.  Together with \eqref{eq:depth-height},
        \begin{align*}
          \max_{0 \le i \le 2(n-1)} s_i
          &= \max_{0 \le i \le 2(n-1)} \dep(u_i)
          = \max_{v \in V(G_T)} \dep(v)
          = \hgt(G_T). \qedhere
        \end{align*}
    \end{proof}

    \begin{proof}[Proof of \autoref{prop:widthpathtree}]
        Consider $l$ with $1 \le l \le \hgt(G_T)$, and define
        \begin{align*}
          A_l &:= \{v \in V(G_T) : \dep(v) = l\},\\
          B_l &:= \{i : s_i = l,\; s_i - s_{i-1} = 1\}.
        \end{align*}
        We construct a bijection $\varphi_l : A_l \to B_l$.  For each $v \in A_l$, DFS first arrives at $v$ from its parent via a unique $\UU$ step; let $i_v$ denote the index of this step.  By \autoref{prop:heightpathtree}, $s_{i_v} = \dep(u_{i_v}) = \dep(v) = l$, and by construction $s_{i_v} - s_{i_v - 1} = 1$, so $i_v \in B_l$.  Set $\varphi_l(v) := i_v$.
        
        \textit{Injectivity.}  DFS visits each node for the first time at a distinct step, so $v \ne w$ implies $i_v \ne i_w$.
         
        \textit{Surjectivity.}  Every $i \in B_l$ is a $\UU$ step satisfying $s_i = l$, so the pointer moves to $u_i$ with $\dep(u_i) = l$ by \autoref{prop:heightpathtree}.  Since DFS first enters any node via a downward step, and step $i$ is the first occasion $u_i$ is reached (a second descent to $u_i$ would require a prior ascent from it, which would have been recorded as the unique $\DD$ step terminating that visit), we have $i = i_{u_i} = \varphi_l(u_i)$.
         
        Hence $|A_l| = |B_l|$ for every $1 \le l \le \hgt(G_T)$.  It remains to verify that excluding the level $l = 0$ does not alter the maximum.  The root is the unique node at depth $0$, so $|A_0| = 1$.  If $\wid(G_T) = 1$, then $|A_l| = 1$ for all $l$, and since $\hgt(G_T) \ge 1$ for $n \ge 2$ the set $\{A_l : 1 \le l \le \hgt(G_T)\}$ is non-empty with maximum $1 = \wid(G_T)$. If $\wid(G_T) \ge 2$, the maximum is attained at some $l \ge 1$ and $|A_0| = 1$ is strictly dominated.  In both cases,
        \begin{align*}
          \wid(G_T)
          &= \max_{0 \le l \le \hgt(G_T)} |A_l|
           = \max_{1 \le l \le \hgt(G_T)} |A_l| 
           = \max_{0 < l \le \hgt(G_T)} |B_l|. \qedhere
        \end{align*}
    \end{proof}

    \begin{proof}[Proof of Correctness of \autoref{algo:pathtotree}]
        We first address two implicit invariants. First, the branch on line~8 is unreachable when $i = 0$: if the current character were $\DD$ at depth zero, the lattice path would descend below zero, contradicting the non-negativity condition of a Dyck path. Second, whenever the algorithm reaches line~8, it is guaranteed that a unique $j < nl$ satisfying $A_w[j, nl] = 1$ exists. Indeed, a step $\DD$ at position $t$ implies $s_t - s_{t-1} = -1$, so the path is closing a previously opened step $\UU$; the corresponding upward transition was handled by line~6 at the time that $\UU$ was processed, which set $A_w[i, nl] \gets 1$ for the unique parent $i$ of $nl$ at that moment. The uniqueness of such $j$ follows from the observation that node $nl$ is introduced exactly once, at the single $\UU$ step that executes line~6 for $nl$, so no index $j^\prime \ne j$ can satisfy $A_w[j^\prime, nl] = 1$ at the time line~8, since every previous step $\UU$ operates on a strictly lower value of $nl$. Finally, $\psi$ is injective: if $\psi(w) = \psi(w^\prime)$ for two Dyck words $w, w^\prime \in \mathcal{D}_{m-1}$, then applying $\phi$ to both sides and invoking the \autoref{lem:bijectionpathtree} yields $w = \phi(\psi(w)) = \phi(\psi(w^\prime)) = w^\prime$. Hence, the algorithm terminates correctly on every valid Dyck word, the output $A_w$ is the adjacency matrix of $\psi(w)$, and $\psi$ is a well-defined injection in $\mathcal{D}_{m-1}$.
    \end{proof}

\section{Preliminaries}
\label{app:prelim}
    \begin{fact}
        Let $v \in \R^n$ be an input vector, $W \in \R^{m \times n}$ be a weight matrix, and $c \in \R^m$ be a bias vector. An affine transformation defined by $v^\prime = Wv + c$ can be rewritten strictly as a linear transformation $v^\prime = W^\prime \tilde{v}$ by concatenating the input vector $v$ with $c$.
    \end{fact}
    
    \begin{proof}
        If we explicitly want to define our new input vector by concatenating $v$ and $c$, we can write $\tilde{v} = \begin{bmatrix} v^\top \parallel c^\top \end{bmatrix}^\top \in \R^{n+m}$.
    
        Define $W^\prime \in \R^{m \times (n+m)}$ as $W^\prime = \begin{bmatrix} W \parallel I_m \end{bmatrix}$. Then
        \begin{align*}
            W^\prime (\tilde{v}) &= \begin{bmatrix} W \parallel I_m \end{bmatrix} \begin{bmatrix} v \\ c \end{bmatrix} = Wv + I_m c\\
            &= Wv + c = v^\prime. \qedhere
        \end{align*}
    \end{proof}
    
    \begin{fact}
        Let $X$ and $U$ be vector spaces over a field $\R$. A linear transformation $L: X \to U$ parameterized by a matrix $W \in \R^{m \times n}$ such that $L(v) = Wv$ can be represented as a bilinear operation.
    \end{fact}
    
    \begin{proof}
        To represent the linear transformation as a bilinear-style operation where the matrix weights are treated as parameters and $v$ is the input, we consider the second argument as $\R$. 
    
        We define a mapping $B: X \times \R \to U$, as $B(v, c) = Wv \cdot c$, where $v \in V$, $c \in \R$. To show this behaves linearly with respect to the input vector space $X$, let $v_1, v_2 \in X$ be input vectors, and let $\alpha, \beta \in \R$ be scalars.
        \begin{align*}
            B(\alpha v_1 + \beta v_2, c) &= W(\alpha v_1 + \beta v_2) \cdot c\\
            &= (\alpha W v_1 + \beta W v_2) \cdot c\\
            &= \alpha (W v_1 \cdot c) + \beta (W v_2 \cdot c)\\
            &= \alpha B(v_1,c) + \beta B(v_2,c).
        \end{align*}
        This satisfies linearity in the first argument. Similarly for the second coordinate
        \begin{align*}
            B(v, \alpha c_1 + \beta c_2) &= Wv(\alpha c_1 + \beta c_2)\\
            &= Wv\alpha c_1 + Wv\beta c_2\\
            &= \alpha Wv \cdot c_1 + \beta Wv \cdot c_2\\
            &= \alpha B(v,c_1) + \beta B(v,c_2). \qedhere
        \end{align*}
    \end{proof}

    \begin{lemma}
    \label{lem:bilinearstack}
        Let $U$ be a vector space and let $U^\prime, U^{\prime\prime}$ be vector spaces with $W = U^\prime \oplus U^{\prime\prime}$. Let $B_1 : U \times U \to U^\prime$ and $B_2 : U \times U \to U^{\prime\prime}$ be bilinear maps, and define $B : U \times U \to W$ by
        \[
            B(v, w) = \begin{pmatrix} B_1(v, w) \\ B_2(v, w) \end{pmatrix}.
        \]
        Then $B$ is bilinear.
    \end{lemma}
    
    \begin{proof}
        We verify linearity in the first argument; the second is identical. Let $v, \tilde v, w \in U$ and let $c$ be a scalar.
    
        By the bilinearity of $B_1$ and $B_2$ in their first argument,
        \[
            B_i(v + \tilde v, w) = B_i(v, w) + B_i(\tilde v, w), \qquad i \in \{1, 2\}.
        \]
        Since addition in $W = U^\prime \oplus U^{\prime\prime}$ is taken componentwise,
        \begin{align*}
            B(v + \tilde v, w)
            &= \begin{pmatrix} B_1(v, w) + B_1(\tilde v, w) \\ B_2(v, w) + B_2(\tilde v, w) \end{pmatrix}\\
            &= \begin{pmatrix} B_1(v, w) \\ B_2(v, w) \end{pmatrix} + \begin{pmatrix} B_1(\tilde v, w) \\ B_2(\tilde v, w) \end{pmatrix}\\
            &= B(v, w) + B(\tilde v, w).
        \end{align*}
        By the homogeneity of $B_1$ and $B_2$ in their first argument, $B_i(c v, w) = c\,B_i(v, w)$, and since scalar multiplication in $W$ is componentwise,
        \begin{align*}
            B(c v, w)
            &= \begin{pmatrix} c\,B_1(v, w) \\ c\,B_2(v, w) \end{pmatrix} = c \begin{pmatrix} B_1(v, w) \\ B_2(v, w) \end{pmatrix}\\
            &= c\,B(v, w).
        \end{align*}
        The second argument is treated identically, using the bilinearity of each $B_i$ there. Hence $B$ is bilinear.
    \end{proof}

    Before we present the explicit constructions we substantiate the \autoref{def:algorelationbycot}.\\
    \textbf{Realization of Algorithms via Chain-of-Thought.} The construction of Transformers that implement functions and algorithmic dynamics has been extensively studied in the expressivity literature \citep{giannou2023looped, zhao2023transformers, dutta2025existence}. To ground the constructions presented in \autoref{sec:prelim}, \autoref{def:algorelationbycot} outline the general scheme we adopt. The input sequence $(x_0, \ldots, x_{m-1}) \in (\mathbb{R}^d)^m$ encodes the graph structure, where each $x_i$ may additionally carry a fixed number of auxiliary coordinates reserved for intermediate computation. The initial token $y_0 \in \R^d$ serves as the designated start symbol for each algorithm; for instance, in graph traversal, a fixed block of coordinates within $y_0$ encodes the identity of the source vertex. However, during the design of individual algorithms, including those that take the lattice path as input, we will specify the encoding required to realize particular algorithms via CoT.   
    
    We call the realization \textit{uniform} if the map from $n$ to the family of resultant decoders is logspace-computable. In our work, all constructions follow uniformity in this sense. Note that the encoding $\operatorname{enc}$ represents an interpretation of the state of the algorithm $\mathcal{A}$. While $\operatorname{enc}$ is not spelled out in each construction, its role is implicit and can be read from context. For instance, in the Dijkstra construction the block $\pi_{\dis}$ carries the distance $\dis$, which serves as $\operatorname{enc}(\mathsf{s}_t)$; by contrast, in the DFS construction only the order of traversal matters, and no such $\operatorname{enc}(\mathsf{s}_t)$ is present. The remaining cases follow the same pattern.    

\section{DFS}
\label{app:cotdfs}
    \begin{theorem}
    \label{thm:cotdfs}
        A two-layer, two-head Transformer decoder can simulate the depth-first traversal $\Dfs$ on any simple directed graph $G = (V, E, A)$ in $O(|V| + |E|)$ chain-of-thought steps.
    \end{theorem}
    \begin{proof}
        We show that the $t$\textsuperscript{th} step of a two-layer, two-head Transformer decoder implements $\dfs^{(t)}(s)$. Let $G = (V, E, A)$ be a graph with adjacency matrix $A \in \{0,1\}^{n \times n}$, and denote by $e_i \in \{0,1\}^n$ the standard basis vector corresponding to vertex $v_i \in V$.
    
        The initial token sequence consists of embedded vectors $x_0, \ldots, x_{n-1} \in \R^d$ with $d = 6n + 3$, preceded by a sentinel token $x_\bot = 0^d$. Each embedding is a concatenation of nine blocks, that together capture the instantaneous description of the traversal: the current vertex, the parent vertex, the visited marker, the neighborhood, a temporary register, a buffer register, two state-control flags, and a type flag. We define the node embedding $x_i \in \R^{6n+3}$ for each $v_i \in V$ as
        \begin{align*}
            x_i &= \left[ e_i \parallel 0^n \parallel 0^n \parallel A_{i,:} \parallel 0^n \parallel 0^n \parallel {-1} \parallel {-1} \parallel 0 \right]^\top,
        \end{align*}
        where the block $A_{i,:}$ is the $i$\textsuperscript{th} row of the adjacency matrix, encoding the neighborhood $\nei(v_i)$. Note that the collection $\{x_i\}_{i=0}^{n-1}$ uniquely identifies $G$. The embedding carries three scalar flags. A state-control flag $flg_1 \in \{1, 0, -1\}$ indicates whether the current step is a \emph{traverse}, \emph{backtrack}, or \emph{idle} operation, respectively. An auxiliary flag $flg_2$ supports the final computation of $flg_1$. A type flag $flg_3 \in \{0,1\}$ distinguishes $x_i$s ($flg_3 = 0$) from $y_t$s ($flg_3 = 1$). The dynamic $\Dfs$ at step $t$ is maintained in vector $y_t \in \R^{6n+3}$, where $t \leq O(|V|+|E|)$. Given a source vertex $s = v_p \in V$, the initial state $y_0$ is
        \begin{align*}
            y_0 &= \left[ e_p \parallel 0^n \parallel e_p \parallel A_{p,:} \parallel 0^n \parallel 0^n \parallel {-1} \parallel {-1} \parallel 1 \right]^\top.
        \end{align*}
        The two embeddings differ in precisely two blocks: $y_0$ records $e_p$ in the visited restriction and sets $flg_3 = 1$, reflecting that $s$ has been visited and that $y_0$ is a traversal token.
    
        Let $\mathcal{E} \subset \{0,1\}^n$ denote the set of standard basis vectors in $\R^n$. For notational convenience, we define the following selectors from $\R^d$ to access individual blocks of an embedding. The maps $\pi_{cur}, \pi_{par}: \R^d \to \mathcal{E}$ extract the current and parent node indicators\footnote{At initialization, $\pi_{par}(\alpha) = 0^n$ for every $\alpha \in \{x_0, \ldots, x_{n-1}, y_0\}$, reflecting that $par(v_i) = \phi$ for all $v_i \in V \cup \{s\}$.}; the maps $\pi_{vis}, \pi_{nbr}: \R^d \to \{0,1\}^n$ extract the visited and neighborhood vectors; the maps $\pi_{tmp}, \pi_{buf}: \R^d \to \R^n$ extract the temporary and buffer registers; and $\pi_{flg_k}: \R^d \to \R$ for $k \in \{1,2,3\}$ extracts the corresponding control flag.
    
        We proceed inductively. Suppose $y_t$ encodes the instantaneous description of $G$ immediately before $\Dfs^{(t+1)}(s)$. The first operation updates the control flags $\pi_{flg_1}(y_t)$ and $\pi_{flg_2}(y_t)$; the attention in this layer is trivial, with $W_O^{(1)} = \mathbf{0}^{d \times d}$. The update is carried out by the FFN block. The bilinear map $W_1^{(1)}$ computes the inner product $\gamma_t = (\mathbf{1} - \pi_{vis}(y_t))^\top \, \pi_{nbr}(y_t)$, which counts the number of unvisited neighbors of the current vertex, together with its unit-shifted counterpart $\gamma_t - 1$, placing these into the $flg_1$ and $flg_2$ positions respectively. After the $\relu$ activation, the subsequent linear map $W_2^{(1)}$ writes $ \pi_{flg_1}(y_t) = \relu(\gamma_t) - \relu(\gamma_t - 1)$, $\pi_{flg_2}(y_t) = 0$, and copies $\pi_{cur}(y_t)$ into the buffer register. We denote the resulting vector by $y_t^{(1)}$. Observe that $\pi_{flg_1}(y_t^{(1)})$ is $1$ precisely when $\gamma_t \geq 1$ and $0$ otherwise, encoding the \emph{traverse}-or-\emph{backtrack} decision. After this step, all preceding vectors $\{x^{(1)}_\bot, x^{(1)}_0, \ldots, x^{(1)}_{n-1}, y^{(1)}_0, \ldots, y^{(1)}_{t-1}\}$ differ from their initial values in the $flg_1$, $flg_2$, and buffer positions. Because none of these positions of the aforementioned vector interact with any operation in the subsequent layer, this modification does not affect the computations.
    
        The second layer comprises two mutually exclusive attention heads. The first head executes the operation \emph{traverse}. The query $Q_T^{(2)} y_t^{(1)}$ is a bilinear map \citep{rizvi2024simulating} implementing the Hadamard product $(\mathbf{1} - \pi_{vis}(y_t^{(1)})) \odot \pi_{nbr}(y_t^{(1)})$, which masks out all previously visited vertices from the neighborhood. Evaluated against keys $K_T^{(2)} \alpha = \pi_{cur}(\alpha)$ for all $\alpha \in \{x_\bot^{(1)}, x_0^{(1)}, \ldots, x_{n-1}^{(1)}, y_0^{(1)}, \ldots, y_t^{(1)}\}$, the UHA resolves the lowest-indexed token $\alpha_T^\prime$ (among $\{x_i^{(1)}\}$) whose current-node encoding $\pi_{cur}(\alpha_T^\prime)$ corresponds to $\dfs^{(t+1)}(s)$. When $\pi_{flg_1}(y_t^{(1)}) = 0$, $\alpha_T^\prime$ defaults to $x_\bot^{(1)}$.
        
        The second head implements the operation \emph{backtrack}. The query $Q_B^{(2)} y_t^{(1)}$ is a bilinear map computing $(1 - \pi_{flg_1}(y_t^{(1)})) \cdot [\pi_{par}(y_t^{(1)}) \parallel \pi_{flg_3}(y_t^{(1)})]$, which activates only when $\pi_{flg_1}(y_t^{(1)}) = 0$. Evaluated against keys $K_B^{(2)} \alpha = [\pi_{cur}(\alpha) \parallel \pi_{flg_3}(\alpha)]$ for all $\alpha \in \{x_\bot^{(1)}, x_0^{(1)}, \ldots, x_{n-1}^{(1)}, y_0^{(1)}, \ldots, y_t^{(1)}\}$, the UHA resolves the lowest-indexed token $\alpha^{\prime}_B$ (among $\{y_i^{(1)}\}$) such that $\pi_{cur}(\alpha^{\prime}_B)$ corresponds to the vertex first visited during $\dfs^{(t^{\prime})}(s)$ for some $0 \leq t^{\prime} < t$. When $\pi_{flg_1}(y_t^{(1)}) = 1$, the query vanishes and $\alpha^{\prime}_B$ defaults to $x_\bot^{(1)}$.  
        The linear map $W_O^{(2)} \in \{-1, 0, 1\}^{d\times 2d}$ produces the combined update $\alpha^{\prime\prime} = \alpha^{\prime\prime}_T + \alpha^{\prime\prime}_B$, where 
        \begin{align*}
            \alpha_T^{\prime\prime} &= \left[{-\pi_{cur}(\alpha_T^\prime)} \parallel {-\pi_{par}(\alpha_T^\prime)} \parallel \pi_{cur}(\alpha_T^\prime) \parallel 0^n \parallel \pi_{nbr}(\alpha_T^\prime) \parallel 0^n \parallel 0^3 \right]^\top \quad \text{ and }\\
            \alpha^{\prime\prime}_B &= \left[{-\pi_{cur}(\alpha^{\prime}_B)} \parallel {-\pi_{par}(\alpha^{\prime}_B)} \parallel 0^{2n} \parallel \pi_{nbr}(\alpha^{\prime}_B) \parallel 0^n \parallel 0^3 \right]^\top.
        \end{align*}
        Since the two heads are mutually exclusive {--} governed by $\pi_{flg_1}(y_t^{(1)})$ {--} exactly one of $\alpha^{\prime\prime}_T$ and $\alpha^{\prime\prime}_B$ is $0^{6n+3}$. 
        
        Since $G$ is simple, during \emph{traverse} the current vertex and any unvisited neighbor are necessarily distinct, whence $\pi_{cur}(y_t^{(1)})^\top \pi_{cur}(\alpha^{\prime\prime}_T) = 0$. Similarly, during \emph{backtrack}, $\pi_{par}(y_t^{(1)})^\top \pi_{par}(\alpha^{\prime\prime}_B) = 0$ holds strictly: were it otherwise, we would have $\pi_{cur}(\alpha^{\prime\prime}_B) = \pi_{par}(\alpha^{\prime\prime}_B)$, implying a self-loop and contradicting the assumption that $G$ is simple.
    
        The linear map $W_1^{(2)}$ then negates the current and parent restrictions of $\alpha^{\prime\prime} + y_t^{(1)}$, resets the temporary register, and copies the (previous) temporary register into the neighborhood position, leaving all other positions intact. The subsequent $\relu$ activation discards the negated priors: $-\pi_{cur}(y_t^{(1)})$ is eliminated from the current restriction in both traversal and backtrack, while $-\pi_{par}(y_t^{(1)})$ is eliminated from the parent restriction during backtrack. Denote the result by $\beta$. At this stage, $\beta$ agrees with the desired output $y_{t+1}$ of the $(t+1)$-th $\Dfs$ dynamic in every restriction except the parent restriction during \emph{traverse}. To resolve this, the bilinear map $W_2^{(2)}$ sets $\pi_{par}(\beta)$ to value $\pi_{flg_1}(\beta) \pi_{buf}(\beta) + (1 - \pi_{flg_1}(\beta)) \pi_{par}(\beta)$, which selects the buffered vertex (the pre-traversal current node) when $\pi_{flg_1}(y_t) = 1$ and retains the existing parent otherwise. The chain-of-thought terminates when $\pi_{cur}(y_{t+1}) = 0^n$.
    \end{proof}

\section{Dijkstra Algorithm}
\label{app:cotdij}
    \begin{theorem}
    \label{thm:cotdijkstra}
        A two-layer, single-head Transformer decoder can simulate the Dijkstra algorithm $\dij$ on any simple connected directed graph $G = (V, E, A)$ with $\mathsf{A} = \R_{> 0}$ in exactly $|V|-1$ chain-of-thought steps.
    \end{theorem}
    \begin{proof}
        We show that the $t$\textsuperscript{th} step of a two-layer single-head Transformer decoder implements $\dij^{(t)}$. Let $G = (V, E, A)$ be a graph with adjacency matrix $A \in \R_{> 0}^{n \times n}$, fix a source $s = v_p \in V$, and write $e_i \in \{0,1\}^n$ for the standard basis vector corresponding to $v_i$. The input token sequence consists of embeddings $x_0, \ldots, x_{n-1} \in \R^d$ with $d = 5n+1$ that together encode the instantaneous description of the traversal: the identity of the current vertex, the tentative distance from $s$ to that vertex, the visited indicator, the outgoing edge weights, the global distance vector $\dis$, and a buffer. Concretely, the embedding of vertex $v_i$ is
        \[
            x_i = \left[e_i \parallel -\lambda \parallel 0^n \parallel A_{i,:} \parallel \lambda^n \parallel 0^n\right]^\top,
        \]
        where $A_{i,:}$ is the $i$\textsuperscript{th} row of the adjacency matrix encoding the outgoing weights from $v_i$, and $\lambda$ is a sufficiently large constant. The scalar $-\lambda$ in the second block serves as a sentinel indicating that no finite shortest-path estimate has yet been assigned; the fourth block will be repurposed during output token generation, and the final block acts as a scratchpad. Note that the collection $\{x_i\}_{i=0}^{n-1}$ uniquely determines $G$. The dynamic $\dij$ at step $t \in \{0, \ldots, n-1\}$ is maintained in a vector $y_t \in \R^d$; given source $s = v_p$, the initial state is
        \[
            y_0 = \left[e_p \parallel 0 \parallel e_p \parallel 0^n \parallel \lambda^n \parallel 0^n\right]^\top.
        \]    
        
        That is, $y_0$ starts the CoT with the initialized state $\dij^{(0)}$ {--} source visited, distance zero {--} whereas the embeddings $\{x_i\}$ encode the graph structure. Let $\mathcal{E} \subset \{0,1\}^n$ denote the set of standard basis vectors in $\R^n$. For notational convenience, we define the following selectors that extract individual blocks from any embedding in $\R^d$:
        $
          \pi_{cur}: \R^d \to \mathcal{E}, \pi_{crd}: \R^d \to \R_{\ge 0} \cup \{-\lambda\}, \pi_{vis}: \R^d \to \{0,1\}^n, \pi_{wei}, \pi_{dis}: \R^d \to \R^n, \text{ and } \pi_{buf}: \R^d \to \R^n.
        $
        In the generated sequence $\{y_{t+1}\}$, the scalar $\pi_{crd}(y_{t+1})$ records $\dis^{(t+1)}[j^\prime]$ for the uniquely newly visited vertex, i.e., the vertex $v_{j^\prime}$ satisfying $vst^{(t+1)}(v_{j^\prime}) = 1$ and $vst^{(t)}(v_{j^\prime}) = 0$. The block $\pi_{wei}(y_{t+1}) = 0^n$ is not used to carry adjacency information in the output tokens and is instead repurposed as another buffer for storing intermediate results between layers. Note that the codomain $\R_{\ge 0} \cup \{-\lambda\}$ reflects a dual role: on input tokens $\{x_i\}$ the selector $\pi_{crd}$ returns the sentinel $-\lambda$, while on generated tokens $\{y_t\}$ it returns a non-negative shortest-path distance.
    
        We proceed inductively. Suppose $y_t$ encodes the instantaneous description of $G$ immediately prior to $\dij^{(t+1)}$; the first layer updates the distance vector from $\dis^{(t)}$ to $\dis^{(t+1)}$. The query $Q^{(1)} y_t$ extracts $\pi_{cur}(y_t)$, the one-hot identity of the current vertex, and is evaluated against keys $K^{(1)}\alpha = \pi_{cur}(\alpha)$ for every token $\alpha \in \{x_0, \ldots, x_{n-1}, y_0, \ldots, y_t\}$; the UHA selects the lowest-indexed match, which is necessarily the input embedding $x_i$ corresponding to $cur^{(t)} = v_i$, thereby getting access of the adjacency row $A_{i,:}$. Taking $V^{(1)}$ as identity, the output projection $W_O^{(1)}$ zeroes out all blocks except $\pi_{wei}$, so that after the residual connection the weight block carries $A_{i,:}$ {--} in particular, this is where $\pi_{wei}$ transitions from its role as a zero buffer on output tokens to carrying adjacency data within the layer. Let this vector be $\alpha^{\prime\prime}$. The first linear map of the feed-forward network, $W_1^{(1)}$, then computes $\pi_{dis}(\alpha^{\prime\prime} + y_t)[i] - \pi_{crd}(\alpha^{\prime\prime} + y_t) - \pi_{wei}(\alpha^{\prime\prime} + y_t)[i]$ for each $i \in \{0, \ldots, n-1\}$; after the $\relu$ activation, $\pi_{wei}$ stores the non-negative updates corresponding to $\dis^{(t+1)}$ with all other positions intact. Denote the result by $\beta$ and observe that $\pi_{wei}(\beta) \ge 0^n$. The second linear map $W_2^{(1)}$ replaces $\pi_{dis}(\beta)$ with $\pi_{dis}(\beta) - \pi_{wei}(\beta)$ and resets $\pi_{wei}(\beta)$ to $0^n$. We denote the resulting vector by $y_t^{(1)}$.
    
        The second layer selects the unvisited vertex of minimum tentative distance. The query $Q^{(2)}$ is a bilinear map \citep{rizvi2024simulating} realizing the Hadamard product $(1 - \pi_{vis}(y_t^{(1)})) \odot (\lambda - \pi_{dis}(y_t^{(1)}))$, which assigns each unvisited vertex $v_j$ a score of $\lambda - \pi_{dis}(y_t^{(1)})[j] > 0$. Evaluated against keys $K^{(2)}\alpha = \pi_{cur}(\alpha)$ for all $\alpha \in \{x_0^{(1)}, \ldots, x_{n-1}^{(1)}, y_0^{(1)}, \ldots, y_t^{(1)}\}$, UHA selects the lowest-indexed token $\alpha^\prime$ among $\{x_i^{(1)}\}_{i=0}^{n-1}$ at which $-\pi_{dis}(y_t^{(1)})$ is maximized over unvisited vertices {--} equivalently, the vertex $v_{j^\prime}$ such that $j^\prime = \argmin\{\dis^{(t+1)}[j] \mid vst^{(t)}(v_j) = 0\}$. The output projection $W_O^{(2)}$ produces
        \begin{align*}
            \alpha^{\prime\prime} &= \left[-\pi_{cur}(\alpha^\prime) \parallel 0 \parallel \pi_{cur}(\alpha^\prime) \parallel 0^n \parallel 0^n \parallel \pi_{cur}(\alpha^\prime)\right]^\top,
        \end{align*}
        so that the residual connection yields $\pi_{cur}(\alpha^{\prime\prime} + y_t^{(1)}) = \pi_{cur}(y_t^{(1)}) - \pi_{cur}(\alpha^\prime)$ and $\pi_{vis}(\alpha^{\prime\prime} + y_t^{(1)}) = \pi_{vis}(y_t^{(1)}) + \pi_{cur}(\alpha^\prime)$, while $\pi_{buf}$ now holds $\pi_{cur}(\alpha^\prime)$. We claim $\pi_{cur}(\alpha^\prime)^\top \pi_{cur}(y_t^{(1)}) = 0$: were this not so, the current vertex would be attending to itself, implying a self-loop and contradicting the assumption that $G$ is simple. Hence, the residual updates the current block to the difference of two orthogonal one-hot vectors and augments the visited set by exactly one element. The feed-forward network of the second layer completes the transition. The linear map $W_1^{(2)}$ acts on $\alpha^{\prime\prime} + y_t^{(1)}$ at two blocks while leaving all others intact: it negates the current block and sets the buffer to $\gamma \in \R^n$ defined coordinate-wise by
        \begin{align*}
            \gamma[i] = &\pi_{dis}(\alpha^{\prime\prime} + y_t^{(1)})[i] + \lambda\pi_{buf}(\alpha^{\prime\prime} + y_t^{(1)})[i] - \lambda\sum_{j=0}^{n-1}\pi_{buf}(\alpha^{\prime\prime} + y_t^{(1)})[j],
        \end{align*}
        for each $i \in \{0, \ldots, n-1\}$. Since
        \begin{equation*}
            \pi_{buf}(\alpha^{\prime\prime} + y_t^{(1)}) = \pi_{cur}(\alpha^\prime) = e_{j^\prime},
        \end{equation*}
        the sum equals unity and the expression simplifies: at coordinate $j^\prime$ the $\lambda$ terms cancel, yielding $\gamma[j^\prime] = \dis^{(t+1)}[j^\prime]$, while at every other coordinate $\gamma[i] = \dis^{(t+1)}[i] - \lambda < 0$ for $\lambda$ sufficiently large. The subsequent $\relu$ therefore retains $\dis^{(t+1)}[j^\prime]$ at the $j^\prime$\textsuperscript{th} coordinate of the buffer and zeros out all others, while simultaneously eliminating $-\pi_{cur}(y_t^{(1)})$ from the current block, leaving $\pi_{cur}(\alpha^\prime)$ as the sole survivor. Denote the result by ${y_t^{(1)}}^\prime$. Finally, $W_2^{(2)}$ sets $\pi_{crd}(y_{t+1}) = \sum_{i=0}^{n-1} \pi_{buf}({y_t^{(1)}}^\prime)[i] = \dis^{(t+1)}[j^\prime]$ and resets $\pi_{buf}(y_{t+1}) = 0^n$, resulting $y_{t+1}$. After exactly $n - 1$ transitions the visited indicator satisfies $\pi_{vis}(y_{n-1}) = 1^n$, which serves as the halting criterion.   
    \end{proof}

\section{Realizing \autoref{algo:pathtotree}}
\label{app:cotpathtotree}
    \begin{theorem}
    \label{thm:cotpathtotree}
        A single-layer, two-head Transformer decoder can simulate the \autoref{algo:pathtotree} in the $2(m-1)$ CoT-steps.
    \end{theorem}
    \begin{proof}
        The $t$\textsuperscript{th} decoding step reconstructs the partial tree built after reading the lattice path $\{(\tau, s_{\tau})\}_{\tau=0}^{t}$, which is one iteration of the \texttt{for} loop on line 3. Let $\mathcal{E}_m = \{e_0, \ldots, e_{m-1}\}$ and $\mathcal{E}_{2m-2} = \{\hat{e}_0, \ldots, \hat{e}_{2m-3}\}$ denote the standard bases of $\R^m$ and $\R^{2m-2}$. The construction maintains $i$, $nl$, and $A_w$ across iterations, encoded in eight blocks of dimension $d = m^2 + 6m$: the current vertex $cur$ (the iterator $i$ of line 2), the label $nl$ of the most recently created vertex, the parent $par$ recorded when an $\UU$ creates a vertex, the flattened adjacency matrix $A$ of the partial tree, the cursor $pos$ marking the position read so far, two scratch blocks $stp$ holding $q_\tau \in \{+1, -1\}$ for $\UU$ resp.\ $\DD$ and on a CoT token $y_t$ is $0$, taking a transient value in $\{+1,-1\}$ only while a step is being read; and $buf$ holding the parent encoding retrieved on a $\DD$ (so that line 9 is realized), and a flag $flg$ separating input tokens $x$ ($flg = 0$) from CoT tokens $y$ ($flg = 1$).

        The input tokens $x_0, \ldots, x_{2m-3} \in \R^d$ encode the Dyck word $w = w_0 \cdots w_{2m-3}$ of the path $\{(\tau, s_{\tau})\}_{\tau=0}^{2m-2}$:
        \begin{align*}
            x_\tau &= [0^m \parallel 0^m \parallel 0^m \parallel 0^{m^2} \parallel \hat{e}_{\tau} \parallel q_{\tau} \parallel 0^{m} \parallel 0]^\top,\\
            q_\tau &= s_{\tau+1} - s_{\tau}.
        \end{align*}
        The dynamic is carried out by tokens $y_t \in \R^d$, with
        \[
            y_0 = [e_0 \parallel e_0 \parallel 0^m \parallel 0^{m^2} \parallel \hat{e}_0 \parallel 0 \parallel 0^m \parallel 1]^\top,
        \]
        so that $i = nl = 0$ ($cur, nl$ at $e_0$) and $par, A_w$ are empty.

        For $v \in \{0,1\}$ and $q \in \{+1,-1\}$,
        \begin{equation}
        \label{eq:proprelu}
            \begin{aligned}
                \relu(v+q-1) = \ind[q=+1] v, \quad \relu(v-q-1) = \ind[q=-1]v,
            \end{aligned}
        \end{equation}
        since $\relu(v+1-1) = v$ while $\relu(v-1-1) = \relu(v-2) = 0$, and symmetrically. Hence, one $\relu$ selects a branch by the sign of $q_t$ and passes a binary $v$ unchanged on the selected branch, zeroing it on the other. Like other constructions, we use the selectors 
        \begin{align*}
            &\pi_{cur}, \pi_{nl} : \R^d \to \mathcal{E}_m, \\
            &\pi_{par}, \pi_{buf} : \R^d \to \mathcal{E}_m \cup \{0^m\}, \\
            &\pi_{A} : \R^d \to \{0,1\}^{m^2}, \\
            &\pi_{pos} : \R^d \to \mathcal{E}_{2m-2}, \\
            &\pi_{stp} : \R^d \to \{+1,-1,0\}, \\
            & \textrm{and} \; \pi_{flg} : \R^d \to \{0,1\},
        \end{align*}
        each extracting its block from an embedding in $\R^d$.
        
        We argue by induction, with the hypothesis that $\pi_A(y_t)$ holds the adjacency matrix of the partial tree built after reading $x_0, \ldots, x_{t-1}$. Step $t$ is driven by two heads, one per branch of line 4: the $\UU$ head is selected when $\pi_{stp}(x_t) = +1$, the $\DD$ head when $\pi_{stp}(x_t) = -1$.

        The $\UU$ head reads the current symbol. With query $Q_{\UU} y_t = \pi_{pos}(y_t)$ and keys $K_{\UU}\alpha = \pi_{pos}(\alpha)$ over $\alpha \in \{x_0, \ldots, x_{2m-3}, y_0, \ldots, y_t\}$, the score $\pi_{pos}(y_t)^\top \pi_{pos}(\alpha)$ is $1$ exactly on the tokens whose block $pos$ is $\hat e_t$, namely $x_t$ and $y_t$; the lowest-indexed of these is the input token $x_t =: \alpha^\prime_{\UU}$. Its value $\alpha^{\prime\prime}_{\UU} := V_{\UU}\alpha^\prime_{\UU}$ copies $\pi_{stp}(\alpha^\prime_{\UU}) = q_t$ into the $stp$ block and is $0$ elsewhere.
        
        The $\DD$ head retrieves a parent. With query $Q_{\DD} y_t = [\pi_{cur}(y_t) \parallel \pi_{flg}(y_t)]$ and keys $K_{\DD}\alpha = [\pi_{cur}(\alpha) \parallel \pi_{flg}(\alpha)]$, the score is $2$ on the $y$ tokens whose current block equals $\pi_{cur}(y_t)$ and at most $1$ on every other token; among these the lowest-indexed is the creation token of the current vertex $i_t$ (the $y$ token that first set $\pi_{cur} = \pi_{cur}(y_t)$ on a $\UU$ step), which precedes $y_t$ and is thus selected over it. Call it $\alpha^\prime_{\DD}$. Its value $\alpha^{\prime\prime}_{\DD} := V_{\DD}\alpha^\prime_{\DD}$ copies the recorded parent $\pi_{par}(\alpha^\prime_{\DD})$ into the $buf$ block and is $0$ elsewhere.
        
        The projection $W_O \in \{0,1\}^{d \times 2d}$ sums the two contributions, $\alpha^{\prime\prime} = \alpha^{\prime\prime}_{\UU} + \alpha^{\prime\prime}_{\DD}$.

        We now construct the FFN. Write $\beta = \alpha^{\prime\prime} + y_t$ for its input; since the heads write only the $stp$ and $buf$ blocks, $\beta$ agrees with $y_t$ on every other block, so $\pi_{cur}(\beta) = \pi_{cur}(y_t)$, $\pi_{nl}(\beta) = \pi_{nl}(y_t)$, $\pi_A(\beta) = \pi_A(y_t)$, while $\pi_{stp}(\beta) = q_t$ and $\pi_{buf}(\beta) = \pi_{par}(\alpha^\prime_{\DD})$. Let $h = W_1\beta \in \R^{d^{\prime}}$ with $d^{\prime} = 2m^2 + 6m - 2 > d$, and overload $\pi^\prime_{\cdot}$ for the blocks of $h$:
        \[
            h = [\pi^{\prime}_{cur} \parallel \pi^{\prime}_{nl} \parallel \pi^{\prime}_{par} \parallel \pi^{\prime}_{buf} \parallel \pi^{\prime}_{A} \parallel \pi^{\prime}_{pos} \parallel \pi^{\prime}_{A^\prime}],
        \]
        which drops the $flg$ and $stp$ blocks and adds $A^\prime$, the indicator of the edge inserted on a $\UU$ (line 6). Let $S \in \{0,1\}^{m \times m}$ be the shift $s_{ij} = \ind[j = i-1]$ (so $S e_k = e_{k+1}$ and $S e_{m-1} = 0$); addition of a scalar to a vector is broadcast, i.e. added to every coordinate. By \eqref{eq:proprelu}, $W_1$ sets each block so that the activation $\relu$ realizes the branch of line 4:
        \begin{equation*}
            \begin{split}
                & \pi^{\prime}_{cur}(h) = S\pi_{nl}(\beta) + \pi_{stp}(\beta) - 1\\
                & \pi^{\prime}_{nl}(h) = \pi_{nl}(\beta) - \pi_{stp}(\beta) - 1\\
                & \pi^{\prime}_{par}(h) = \pi_{cur}(\beta) + \pi_{stp}(\beta) - 1\\
                & \pi^{\prime}_{buf}(h) = \pi_{buf}(\beta) - \pi_{stp}(\beta) - 1\\
                & \pi^{\prime}_{A}(h) = \pi_{A}(\beta)\\
                & \pi^\prime_{A^\prime}(h)_{am+b} = \beta^\top M^{(a,b)} \beta + \pi_{stp}(\beta) - 1 \\
                & \qquad = \pi_{cur}(\beta)_a \cdot \pi_{nl}(\beta)_{b-1} + \pi_{stp}(\beta) - 1\\
                & \pi^{\prime}_{pos}(h) = \pi_{pos}(\beta),
            \end{split}
        \end{equation*}
        for $(a,b) \in \{0,\dots,m-1\}^2$, where $M^{(a,b)} \in \R^{d \times d}$ is the constant matrix with $\beta^\top M^{(a,b)} \beta = \pi_{cur}(\beta)_a \cdot \pi_{nl}(\beta)_{b-1}$, and $\pi_{nl}(\beta)_{b-1}$ is read as $0$ when $b = 0$.
        
        Recall $\pi_{stp}(\beta) = q_t \in \{+1,-1\}$. By \eqref{eq:proprelu}, after activation $\pi^\prime_{cur}(h^\prime)$ equals the shifted label $S\pi_{nl}(\beta)$ on a $\UU$ and $0$ on a $\DD$, while $\pi^\prime_{buf}(h^\prime)$ equals the retrieved parent on a $\DD$ and $0$ on a $\UU$; the remaining blocks gate analogously. Let $S^\prime \in \{0,1\}^{(2m-2)\times(2m-2)}$ be the shift $S^\prime\hat e_\tau = \hat e_{\tau+1}$ with $S^\prime\hat e_{2m-3} = 0$. With $h^\prime = \relu(h)$, the map $W_2 \in \{0,1\}^{d \times d^\prime}$ produces $y_{t+1} = W_2 h^\prime$ by summing the gated blocks:
        \begin{align*}
            & \pi_{cur}(y_{t+1}) = \pi^{\prime}_{cur}(h^\prime) + \pi^{\prime}_{buf}(h^\prime), \\
            & \pi_{nl}(y_{t+1}) = \pi^{\prime}_{cur}(h^\prime) + \pi^{\prime}_{nl}(h^\prime), \\
            & \pi_{par}(y_{t+1}) = \pi^{\prime}_{par}(h^\prime), \\
            & \pi_{A}(y_{t+1})[am+b] = \pi^{\prime}_{A}(h^\prime)[am+b] \\
            & \quad + \pi^{\prime}_{A^\prime}(h^\prime)[am+b] + \pi^{\prime}_{A^\prime}(h^\prime)[bm+a], \\
            & \pi_{pos}(y_{t+1}) = S^\prime\pi^{\prime}_{pos}(h^\prime), \quad
              \pi_{stp}(y_{t+1}) = 0, \\
            & \pi_{buf}(y_{t+1}) = 0^m, \quad
              \pi_{flg}(y_{t+1}) = 1.
        \end{align*}
        On a $\UU$ step this gives $\pi_{cur}(y_{t+1}) = \pi_{nl}(y_{t+1}) = e_{nl_t+1}$ and sets $A[i_t, nl_t+1] = A[nl_t+1, i_t] = 1$ with $\pi_{par}(y_{t+1}) = e_{i_t}$, realizing lines 5--6; on a $\DD$ step it gives $\pi_{cur}(y_{t+1}) = \pi_{buf}(h^\prime) = e_{\operatorname{par}(i_t)}$ while $\pi_{nl}$ and $\pi_A$ are unchanged, realizing line 9. In both cases the cursor advances, $\pi_{pos}(y_{t+1}) = \hat e_{t+1}$.
    \end{proof}

\section{Realizing \texorpdfstring{$\str$}{str} on Trees and Paths}
\label{app:cotstr}
    \begin{theorem}
    \label{thm:cottreestrahler}
        A four-layer, two-head Transformer decoder can compute $\str(G_T)$, the Strahler number of a tree $G_T$ during the simulation of $\Dfs$ in exactly $2n - 1$ CoT steps, where $n$ denotes the number of vertices in $G_T$.
    \end{theorem}
    
    \begin{proof}    
        We show that the $t$\textsuperscript{th} step of a four-layer, two-head decoder, extending that of \autoref{thm:cotdfs} in its embedding and by two further layers, computes the Strahler number $\str(G_T)$ of the tree $G_T$ during $\Dfs$ traversal. We will follow \autoref{obs:dynamicstrdfs}.
    
        We extend the embedding from $\R^{6n+3}$ to $\R^{6n+14}$ by appending ten scalar blocks, written $[\cdots \parallel M \parallel c \parallel M_{buf} \parallel c_{buf} \parallel flg_4 \parallel tmp_{D-1} \parallel tmp_{D} \parallel tmp_{D+1} \parallel tmp_{c_{buf}-1} \parallel tmp_{c_{buf}} \parallel tmp_{c_{buf}+1}]$. In a generated token $y_t$, the block $M$ holds the running maximum of the Strahler numbers of the already-finished children of the current vertex, and $c \in \{0,1,2\}$ counts how many of those children attain $M$. Within a \emph{backtrack} step these blocks are transformed by the layers below before the invariant is restored in $y_{t+1}$. The blocks $M_{buf}$ and $c_{buf}$ are transient buffers, zero on every input and generated token and nonzero only within a step during a \emph{backtrack}. The flag $flg_4 \in \{0, 1, \ldots, 2n-1\}$ is a timestamp, equal to $0$ on every input token and on $y_0$ and to $t$ on the generated token $y_t$, so that the generated tokens are linearly ordered by $flg_4$. The five blocks $tmp_{\cdot}$ are scratch positions used by the later layers; their role is given where they first arise. Whereas the type flag $flg_3$ separates inputs from generated tokens by value in $\{0,1\}$, the flag $flg_4$ records their order; supplied to the backtrack head of \autoref{thm:cotdfs}, it converts that head's leftmost selection into a rightmost one, so that the head resolves the most recent rather than the first occurrence of the parent vertex, as made precise below.
    
        The selectors of \autoref{thm:cotdfs} are retained, and we add $\pi_{M}, \pi_{c}, \pi_{M_{buf}}, \pi_{c_{buf}}, \pi_{flg_4}$ and $\pi_{tmp_{\cdot}} : \R^{d} \to \R$ with $d = 6n + 14$. The input tokens are those of the traversal with the new blocks set to zero,
        \begin{align*}
            x_i &= \left[ e_i \parallel 0^n \parallel 0^n \parallel A_{i,:} \parallel 0^n \parallel 0^n \parallel {-1} \parallel {-1} \parallel 0 \parallel 0^{11} \right]^\top,
        \end{align*}
        so that $\pi_M(x_i) = \pi_c(x_i) = \pi_{M_{buf}}(x_i) = \pi_{c_{buf}}(x_i) = \pi_{flg_4}(x_i) = 0$ for every $v_i \in V$, and the sentinel remains $x_\bot = 0^d$. Given the source $s = v_p$, the initial state is
        \begin{align*}
            y_0 &= \left[ e_p \parallel 0^n \parallel e_p \parallel A_{p,:} \parallel 0^n \parallel 0^n \parallel {-1} \parallel {-1} \parallel 1 \parallel 0^{11} \right]^\top,
        \end{align*}
        which agrees with the traversal's initial state on the inherited blocks and carries an empty accumulator $\pi_M(y_0) = \pi_c(y_0) = 0$, empty buffers $\pi_{M_{buf}}(y_0) = \pi_{c_{buf}}(y_0) = 0$, and timestamp $\pi_{flg_4}(y_0) = 0$. We take the Strahler number of a leaf to be zero.
        
        The selection effected by the backtrack head is the following. With respect to the $t$\textsuperscript{th} query $Q y_t$, the scalar $a_{i,t}$ is the attention score of the key against the token indexed $i$ among $\{x_\bot, x_0, \ldots, x_{n-1}, y_0, \ldots, y_t\}$. Let $j_0 < j_1 < \cdots < j_p$ denote the indices, in ascending order, that jointly attain the maximum score. Where unique hard attention selects the smallest such index $j_0$, the inclusion of $flg_4$ in the key makes the head select the largest, $j_p$, the most recently generated token among those of maximal score.

        The induction follows that of \autoref{thm:cotdfs}, recalled only to the extent the Strahler computation requires. Suppose $y_t$ encodes the instantaneous description immediately before $\Dfs^{(t+1)}(s)$. The first layer leaves its attention trivial and, through the feed-forward block, writes into $\pi_{flg_1}(y_t)$ the indicator $\relu(\gamma_t) - \relu(\gamma_t - 1)$ of whether the current vertex has an unvisited neighbour, where $\gamma_t = (\mathbf{1} - \pi_{vis}(y_t))^\top \pi_{nbr}(y_t)$, so that $\pi_{flg_1}(y_t^{(1)})$ is $1$ on a \emph{traverse} and $0$ on a \emph{backtrack}; it clears $\pi_{flg_2}$ and copies $\pi_{cur}(y_t)$ into the buffer. The eleven new blocks are untouched by this layer. We denote the result $y_t^{(1)}$.
        
        In the second layer the traverse head acts exactly as in \autoref{thm:cotdfs}, resolving the next unvisited neighbour on a \emph{traverse} and defaulting to $x_\bot^{(1)}$ on a \emph{backtrack}. The backtrack head retains its query but is keyed to select the parent's most recent occurrence rather than its first. Its query is the bilinear map
        \[
            Q_B^{(2)} y_t^{(1)} = (1 - \pi_{flg_1}(y_t^{(1)})) \cdot [\pi_{par}(y_t^{(1)}) \parallel \pi_{flg_3}(y_t^{(1)})],
        \]
        which vanishes on a \emph{traverse} and otherwise presents the parent indicator together with the scalar $\pi_{flg_3}(y_t^{(1)}) = 1$. It is evaluated against the keys
        \[
            K_B^{(2)} \alpha = [\lambda \cdot \pi_{cur}(\alpha) \parallel \pi_{flg_4}(\alpha)],
        \]
        for $\alpha \in \{x_\bot^{(1)}, x_0^{(1)}, \ldots, x_{n-1}^{(1)}, y_0^{(1)}, \ldots, y_t^{(1)}\}$, where the constant $\lambda \ge 2n - 1$ exceeds every attainable timestamp. The resulting score $\lambda \cdot \pi_{par}(y_t^{(1)})^\top \pi_{cur}(\alpha) + \pi_{flg_4}(\alpha)$ is lexicographic: the term $\lambda \cdot \pi_{par}(y_t^{(1)})^\top \pi_{cur}(\alpha)$ places every occurrence of the parent strictly above every non-occurrence, since $\lambda$ alone exceeds any timestamp, while the constant $\pi_{flg_3}(y_t^{(1)}) = 1$ admits $\pi_{flg_4}(\alpha)$ into the score with unit weight. The head thus resolves the token $\alpha_B^\prime$ of greatest timestamp among the occurrences of the parent, that is, the parent as it was last left. The query supplies $\pi_{flg_3}$, fixed at $1$, rather than its own $\pi_{flg_4} = t$, which is common to all keys and so cannot order the parent occurrences.
        
        The initial token $y_0$ does not reach the backtrack head. Unless $G$ is the null graph, the root $s$ has an unvisited neighbor, so $\gamma_0 \ge 1$ and $\pi_{flg_1}(y_0^{(1)}) = 1$; the backtrack query vanishes and $y_0$ is served by the traverse head.

        We now specify the value carried by the two heads. On the inherited blocks they act as in \autoref{thm:cotdfs}; we extend only the backtrack value, and only on the blocks $M$ and $c$. Let $\alpha_B^\prime$ be the token resolved by the backtrack head. Its value map $V_B^{(2)}$ moves the accumulator into the buffers and clears the originals,
        \[
        \begin{aligned}
            & \pi_{M_{buf}}(V_B^{(2)} \alpha_B^\prime) = \pi_M(\alpha_B^\prime), \ \  \pi_{c_{buf}}(V_B^{(2)} \alpha_B^\prime) = \pi_c(\alpha_B^\prime), \\
            & \pi_M(V_B^{(2)} \alpha_B^\prime) = 0, \quad
            \pi_c(V_B^{(2)} \alpha_B^\prime) = 0,
        \end{aligned}
        \]
        and is the identity on every other coordinate. The traverse value $\alpha_T^{\prime\prime}$ is zero on all ten new blocks, so the buffers are written by the backtrack head alone, and $M$, $c$, $flg_4$ by neither head.
        
        Form $\beta = \alpha_T^{\prime\prime} + \alpha_B^{\prime\prime} + y_t^{(1)}$. On a \emph{backtrack} the backtrack query does not vanish; the buffers receive the parent accumulator while the head contributes zero to $M$, $c$:
        \[
        \begin{aligned}
            & \pi_{M_{buf}}(\beta) = \pi_M(\alpha_B^\prime), \quad
            \pi_{c_{buf}}(\beta) = \pi_c(\alpha_B^\prime), \\
            & \pi_M(\beta) = \pi_M(y_t^{(1)}), \quad
            \pi_c(\beta) = \pi_c(y_t^{(1)}), \quad
            \pi_{flg_4}(\beta) = t.
        \end{aligned}
        \]
        On a \emph{traverse} the query vanishes, $\alpha_B^\prime$ defaults to $x_\bot^{(1)} = 0^d$, and the buffers stay empty, $\pi_{M_{buf}}(\beta) = \pi_{c_{buf}}(\beta) = 0$, while $M$, $c$ pass through from $y_t^{(1)}$. On a \emph{backtrack} the pair $(\pi_M(\beta), \pi_c(\beta))$ is the accumulator of the vertex being left, $v$ say, and the buffers hold the accumulator of its parent, $u$ say.

        Prior to define the parameters $W_1^{(2)}$, and $W_2^{(2)}$ along with the subsequent layers, we set up the proof for defining the $tmp_{\cdot}$ blocks. Let $D = \str(v) - M_u$ where $M_u = \pi_{M_{buf}}(\beta)$. Similarly, $c_v = \pi_{c}(\beta)$ and $c_u = \pi_{c_{buf}}(\beta)$. With this we want 
        \begin{equation}
        \label{eq:strahM}
            \begin{aligned}
                \pi_{M}(y_{t+1}) &= M_u + \relu(\str(v) - M_u) + \ind[D=0]\ind[c_u>0]
            \end{aligned}
        \end{equation}
        \begin{equation}
        \label{eq:strahc}
            \begin{aligned}
            \pi_{c}(y_{t+1}) &= \ind[D>0]c_u + \ind[D=0]\ind[c_u=0] + \ind[D<0]1
            \end{aligned}
        \end{equation}
        during \emph{backtrack}.
        
        The linear map $W_1^{(2)}$ sets
        \begin{align*}
            & \pi_c(W_1^{(2)}\beta) = \pi_c(\beta) - 1, \\
            & \pi_{flg_4}(W_1^{(2)}\beta) = \pi_{flg_4}(\beta) + 1, \\
            & \pi_{tmp_{c_{buf}-1}}(W_1^{(2)}\beta) = \pi_{c_{buf}}(\beta) - 1,
        \end{align*}
        leaving the other blocks intact; after the activation we write the result $\beta^\prime$. Since $c_v \in \{0,1,2\}$, the rectified count is $\pi_c(\beta^\prime) = \relu(c_v - 1) = \ind[c_v \ge 2]$. The linear map $W_2^{(2)}$ then forms $y_t^{(2)}$ by
        \[
            \pi_M(y_t^{(2)}) = \pi_M(\beta^\prime) + \pi_c(\beta^\prime) = M_v + \ind[c_v \ge 2] = \str(v),
        \]
        and, writing $\pi_M(\beta^\prime) + \pi_c(\beta^\prime) - \pi_{M_{buf}}(\beta^\prime) = \str(v) - M_u = D$,
        \begin{align*}
            & \pi_{tmp_{D-1}}(y_t^{(2)}) = D - 1,\\
            & \pi_{tmp_{D}}(y_t^{(2)}) = D,\\
            & \pi_{tmp_{D+1}}(y_t^{(2)}) = D + 1.
        \end{align*}
        After the second layer the block $M$ carries $\str(v)$, and the blocks $tmp_{D-1}, tmp_D, tmp_{D+1}$ carry $D-1$, $D$, $D+1$.
        
        The third and fourth layers leave the attention trivial, $W_O^{(3)} = W_O^{(4)} = \mathbf{0}$, and act through their feed-forward blocks alone. The linear map $W_1^{(3)}$ extends $y_t^{(2)}$ by six scalar blocks, $1 - D,\ -D,\ -D-1,\ 1 - c_u,\ -c_u, \text{ and } -c_u-1$. The activation reads off indicators through the integer identities
        \begin{align*}
            \ind[v=0] &= \relu(v+1) - 2\relu(v) + \relu(v-1)\\
            \ind[v>0] &= \relu(v) - \relu(v-1).
        \end{align*}
        write the rectified vector $\beta^{\prime\prime}$. The linear map $W_2^{(3)}$ then produces $y_t^{(3)} \in \R^d$, returning the indicators to the $tmp$ blocks:
        \begin{align*}
            & \pi_M(y_t^{(3)}) = \pi_M(\beta^{\prime\prime}) + \pi_{-D}(\beta^{\prime\prime}) \\
            & \qquad = \max(\str(v), M_u),\\
            & \pi_{tmp_{D-1}}(y_t^{(3)}) = \ind[D>0], \\
            & \pi_{tmp_{D}}(y_t^{(3)}) = \ind[D=0], \\
            & \pi_{tmp_{D+1}}(y_t^{(3)}) = \ind[D<0], \\
            & \pi_{tmp_{c_{buf}}}(y_t^{(3)}) = \ind[c_u>0], \\
            & \pi_{tmp_{c_{buf}-1}}(y_t^{(3)}) = \ind[c_u=0].
        \end{align*}
        The $1$s the indicators require are supplied by $\pi_{flg_3} = 1$. The fourth layer carries out the accumulation. Recall that after the third layer the blocks $tmp_{D-1}, tmp_D, tmp_{D+1}$ hold $\ind[D>0], \ind[D=0], \ind[D<0]$ and the blocks $tmp_{c_{buf}}, tmp_{c_{buf}-1}$ hold $\ind[c_u>0], \ind[c_u=0]$. The bilinear map $W_1^{(4)}$ sets
        \begin{align*}
            \pi_M(W_1^{(4)} y_t^{(3)}) &= \pi_M(y_t^{(3)}) + \pi_{tmp_D}(y_t^{(3)})\pi_{tmp_{c_{buf}}}(y_t^{(3)}),\\
            \pi_c(W_1^{(4)} y_t^{(3)}) &= \pi_{tmp_{D-1}}(y_t^{(3)})1 + \pi_{tmp_D}(y_t^{(3)})\pi_{tmp_{c_{buf}-1}}(y_t^{(3)}) + \pi_{tmp_{D+1}}(y_t^{(3)})\pi_{c_{buf}}(y_t^{(3)}),
        \end{align*}
        which, by the contents of those blocks, are $\pi_M(y_t^{(3)}) + \ind[D=0]\ind[c_u>0]$ and $\ind[D>0]c_u + \ind[D=0]\ind[c_u=0] + \ind[D<0]1$, realizing \eqref{eq:strahM} and \eqref{eq:strahc}. Each is a product of two blocks of $y_t^{(3)}$, hence of degree two. The map further resets $M_{buf}$, $c_{buf}$, and the $tmp_{\cdot}$ blocks to zero, leaving the remaining coordinates intact; write the result $\beta^{\prime\prime\prime}$.
        
        The bilinear map $W_2^{(4)}$ corrects $M$ and $c$ for the \emph{traverse} case, assigning
        \begin{align}
        \label{eq:gatingTBstrahlertree}
            &\pi_M(y_{t+1}) = (1 - \pi_{flg_1}(y_t^{(3)})) \pi_M(\beta^{\prime\prime\prime}), \\
            & \pi_c(y_{t+1}) = (1 - \pi_{flg_1}(y_t^{(3)})) \pi_c(\beta^{\prime\prime\prime}),
        \end{align}
        and sets $flg_2 = 0$, $flg_3 = 1$. On a \emph{backtrack} the gate is $1$ and the accumulator commits; on a \emph{traverse} it is $0$, so the freshly visited vertex receives the empty accumulator $(0,0)$. Each coordinate is a product of two blocks, hence of degree two. This is $y_{t+1}$.
        
        At the $(2n-1)$\textsuperscript{th} step the current restriction becomes $0^n$; this final \emph{backtrack} folds $s$ under the same invariant, giving $\pi_M(y_{2n-1}) = \str(s) = \str(G_T)$
    \end{proof}

    \begin{lemma}
    \label{lem:shiftonehot}
        Let $d \ge 2$, and let $v = [m \parallel h \parallel \mathbf{0}_d]^\top \in \R^{2d+1}$ be an input vector where $m \in \{-1, +1\}$, $h \in \{0, 1\}^d$ is a one-hot vector, and $\mathbf{0}_d$ is the zero vector in $\R^d$. Let $h^\prime \in \{0, 1\}^d$ denote the zero-padded, one-position shift of $h$ directed by the sign of $m$ (right-shifted if $m = +1$, left-shifted if $m = -1$). There exists an FFN with $\relu$ activation $W_2 \relu(W_1 v + b_1) + b_2$, that produces $[m \parallel h \parallel h^\prime]^\top$.
    \end{lemma}
    
    \begin{proof}
        Let $I_d$ be the $d \times d$ identity matrix, $\mathbf{1}_d$ be the $d$-dimensional all-ones vector, and $\mathbf{0}_{r \times c}$ be the $r \times c$ zero matrix. Define the right-shift matrix $S_R \in \R^{d \times d}$ such that $(S_R)_{i,j} = 1$ if $i - j = 1$ and $0$ otherwise, and let the left-shift matrix be $S_L = S_R^\top$.
        
        Construct the first linear projection $W_1 \in \R^{(3d+2) \times (2d+1)}$ and bias $b_1 \in \R^{3d+2}$ to partition the hidden representation into five blocks:
        \[
            W_1 = \begin{pmatrix} 
                1 & \mathbf{0}_d^\top & \mathbf{0}_d^\top \\
                -1 & \mathbf{0}_d^\top & \mathbf{0}_d^\top \\
                \mathbf{0}_d & I_d & \mathbf{0}_{d \times d} \\
                0.5 \mathbf{1}_d & S_R & \mathbf{0}_{d \times d} \\
                -0.5 \mathbf{1}_d & S_L & \mathbf{0}_{d \times d}
            \end{pmatrix}, \ \  
            b_1 = \begin{bmatrix} 
                0 \\
                0 \\
                \mathbf{0}_d \\
                -0.5 \mathbf{1}_d \\
                -0.5 \mathbf{1}_d
            \end{bmatrix}
        \]
        Applying $\relu(W_1 v + b_1)$ yields the hidden vector $z = [z_{m,+}, z_{m,-}, z_h, z_+, z_-]^\top$:
        \begin{itemize}
            \item Scalars $z_{m,+} = \relu(m)$ and $z_{m,-} = \relu(-m)$ isolate the sign of $m$.
            \item Since $h \in \{0, 1\}^d$ is non-negative, $z_h = \relu(h) = h$.
            \item The shift gates resolve to mutually exclusive activations: if $m = +1$, $z_+ = S_R h$ and $z_- = \mathbf{0}_d$; if $m = -1$, $z_+ = \mathbf{0}_d$ and $z_- = S_L h$.
        \end{itemize}
        Construct the second linear projection $W_2 \in \R^{(2d+1) \times (3d+2)}$ and bias $b_2 = \mathbf{0}_{2d+1}$ as:
        \[
            W_2 = \begin{pmatrix}
            1 & -1 & \mathbf{0}_d^\top & \mathbf{0}_d^\top & \mathbf{0}_d^\top \\
            \mathbf{0}_d & \mathbf{0}_d & I_d & \mathbf{0}_{d \times d} & \mathbf{0}_{d \times d} \\
            \mathbf{0}_d & \mathbf{0}_d & \mathbf{0}_{d \times d} & I_d & I_d
            \end{pmatrix}
        \]
        Multiplying $W_2 z + b_2$ sums the mutually exclusive shift gates and reconstructs the input blocks:
        \[
        v^\prime = \begin{bmatrix} 
        z_{m,+} - z_{m,-} \\ 
        I_d z_h \\ 
        I_d z_+ + I_d z_- 
        \end{bmatrix} = \begin{bmatrix} m \\ h \\ h^\prime \end{bmatrix}. \qedhere
        \]
    \end{proof}
    
    \begin{lemma}
    \label{lem:mMc}
        For a real bound $\lambda \ge 0$, and let $v = [m \parallel M \parallel c]^\top \in \{-1, 1\} \times [0, \lambda] \times [0, \lambda]$ be the input vector. There exists an FFN with $\relu$ activation $W_2 \relu(W_1 v + b_1) + b_2$, such that the output vector satisfies the following:
        \[
            v^\prime = \begin{cases} 
                [1 \parallel 0 \parallel 0]^\top & \text{if } m = 1, \\ 
                [-1 \parallel M \parallel c]^\top & \text{if } m = -1. 
            \end{cases}
        \]
    \end{lemma}
    
    \begin{proof}
        For any scalar $x \in [0, \lambda]$ and a variable $m \in \{-1, 1\}$, define the scalar gating function:
        \[
            g(m, x) = \relu(x) - \relu\left(x + \frac{\lambda}{2}m - \frac{\lambda}{2}\right).
        \]
        When $m = 1$, the second term simplifies to $\relu(x)$, yielding $g(1 \parallel x) = \relu(x) - \relu(x) = 0$. When $m = -1$, the second term becomes $\relu(x - \lambda) = 0$ because $x \le \lambda$, leaving $g(-1 \parallel x) = \relu(x) = x$ since $x \ge 0$. Furthermore, the identity mapping for the binary variable $m$ can be expressed exactly as $m = \relu(m) - \relu(-m)$.
        
        We construct $h = \relu(W_1 v + b_1)$ by defining the first layer's parameters as:
        \[
            W_1 = \begin{pmatrix} 
                1 & 0 & 0 \\ 
                -1 & 0 & 0 \\ 
                0 & 1 & 0 \\ 
                \frac{\lambda}{2} & 1 & 0 \\ 
                0 & 0 & 1 \\ 
                \frac{\lambda}{2} & 0 & 1 
            \end{pmatrix}, \quad 
            b_1 = \begin{bmatrix} 
                0 \\ 0 \\ 0 \\ -\frac{\lambda}{2} \\ 0 \\ -\frac{\lambda}{2} 
            \end{bmatrix}.
        \]
        Applying the $\relu$ activation yields the hidden state vector:
        \[
            h = \relu\begin{bmatrix} 
                m \\ 
                -m \\ 
                M \\ 
                M + \frac{\lambda}{2}m - \frac{\lambda}{2} \\ 
                c \\ 
                c + \frac{\lambda}{2}m - \frac{\lambda}{2}
            \end{bmatrix}.
        \]
        Setting the output bias to zero ($b_2 = \mathbf{0}$) and defining the output projection matrix $W_2$ as:
        \[
            W_2 = \begin{pmatrix} 
                1 & -1 & 0 & 0 & 0 & 0 \\ 
                0 & 0 & 1 & -1 & 0 & 0 \\ 
                0 & 0 & 0 & 0 & 1 & -1 
            \end{pmatrix},
        \]
        we obtain the network output:
        \[
            \mathcal{N}(v) = \begin{bmatrix} 
                \relu(m) - \relu(-m) \\ 
                g(m \parallel M) \\ 
                g(m \parallel c) 
            \end{bmatrix} = \begin{bmatrix} 
                m \\ 
                g(m \parallel M) \\ 
                g(m \parallel c) 
            \end{bmatrix}.
        \]
        Evaluating this vector at $m = 1$ yields $[1 \parallel 0 \parallel 0]^\top$, and evaluating at $m = -1$ yields $[-1 \parallel M \parallel c]^\top$.
    \end{proof}

    \begin{theorem}
    \label{thm:cotpathstrahler}
        A four-layer, single-head Transformer decoder can compute $\str(\psi(w))$, the Strahler number of a tree corresponding to the Dyck word $w$ in exactly $n$ CoT steps, where $n=|w|$.
    \end{theorem}
    \begin{proof}[Proof Sketch.]
        Like other proofs, we show that the $t$\textsuperscript{th} step of a four-layer single head decoder can calculate $\str(\psi(w))$. The proof will rely on the following observation and \autoref{obs:dynamicstrdfs}.

        \begin{observation}
        \label{obs:strpathtwoMc}
            Let the Dyck word $w = w_0\cdots w_{n-1}$ encodes the lattice path $\{(\tau, s_\tau)\}_{\tau=0}^{n}$. Let $q_\tau = s_{\tau}-s_{\tau-1}$. For a time step $t$, $\str(\psi(w_0\cdots w_t))$ depends on two $(M,c)$-pairs (see \autoref{obs:dynamicstrdfs}) if $q_t = -1$. The first is the $(M_v,c_v)$-pair of the vertex $v$ with co-ordinate $(t, s_t)$, while the second $(M_u, c_u)$ is associated with the vertex $u$ with co-ordinate $(t^\prime, s_{t^\prime})$ such that $\nexists\ t^{\prime\prime}, t^\prime < t^{\prime\prime} < t \text{ and } s_{t^{\prime\prime}} = s_{t^\prime}$.
        \end{observation}

        Based on this, we shall provide the construction up to the second layer alone and the last layer partially. The proof thereon follows \autoref{thm:cottreestrahler}. Assume $\mathcal{E}_n = \{e_0, \ldots, e_{n-1}\}$ denotes the standard basis vector on $\R^{n}$. Let $S \in \{0,1\}^{n \times n}$ be the shift $s_{ij} = \ind[j = i-1]$ (so $S e_k = e_{k+1}$ and $S e_{n-1} = 0$). The input tokens $x_0, \ldots, x_{n-1} \in \R^d$ where $d=3n+15$ encodes the lattice paths $\{(\tau, s_\tau)\}_{\tau=0}^{n}$ as follows:
        \[
            x_\tau = [e_\tau \parallel \tau+1 \parallel 0 \parallel q_{\tau} \parallel 0 \parallel e_{s_\tau} \parallel 0^n \parallel 0^4 \parallel 0 \parallel 0^6]^\top,
        \]
        for $\tau=0$ to $n-1$. The Dynamic is carried out by tokens $y_t \in \R^d$ with 
        \[
            y_0 = [e_0 \parallel 1 \parallel 0 \parallel 1 \parallel 0 \parallel e_0 \parallel 0^n \parallel 0^4 \parallel 1 \parallel 0^6]^\top.
        \]
        The selectors 
        \begin{align*}
            &\pi_{pos} : \R^d \to \R^n, \\
            &\pi_{pos_{sc}}, \pi_{tmp_{pos_{sc}}}, \pi_{tmp_{q}} : \R^d \to \R, \\
            &\pi_{q} : \R^d \to \{-1, +1\}, \\
            &\pi_{ht}, \pi_{tmp_{ht}} : \R^d \to \R^n, \\
            &\pi_{M}, \pi_{c}, \pi_{M_{buf}}, \pi_{c_{buf}} : \R^d \to \R, \\
            &\pi_{flg}: \R^d \to \{0, 1\}, \\
            &\pi_{tmp_{D-1}}, \pi_{tmp_{D}}, \pi_{tmp_{D+1}}, \pi_{tmp_{c_{buf}-1}}, \pi_{tmp_{c_{buf}}}, \\ &\textrm{and} \; \pi_{tmp_{c_{buf}+1}}: \R^d \to \R,
        \end{align*}
        extract the relevant blocks from an embedding in $\R^d$. The first three blocks correspond to the one-hot position $\tau$, its decimal equivalent, and a temporary block. The blocks denoted by $q$ and $tmp_q$ register $q_\tau$ and its temporary computation trace. The following two store the one-hot equivalent of $s_\tau$ and an intermediate buffer. The blocks $M$ and $c$ correspond to the $M_v$ and $c_v$ of \eqref{eq:dynamicstrdfs}. Similarly, $M_{buf}$ and $c_{buf}$ correspond to $M_u$ and $c_u$. The block $flg$ distinguishes the input tokens $x_i$s ($flg=0$) from the generated tokens $y_i$s ($flg=1$), while $tmp_{flg}$ is for temporary calculation. The remaining blocks continue to serve the purpose as in \autoref{thm:cottreestrahler}. We argue that $\str(\psi(w)) = \pi_M(y_{n}) + \ind[\pi_c(y_{n})]$.

        The first layer will update the blocks $pos$, $q$, and $ht$. The query exacts the block $pos$ and shifts the one $Q^{(1)}y_t = S\pi_{pos}(y_t)$ and evaluated against keys $K\alpha = \pi_{pos}(\alpha)$ for all $\alpha \in \{x_0, x_1, \ldots, x_{n-1}, y_0, \ldots, y_t\}$. The UHA selects the lowest-indexed match from $x_i$'s, say $\alpha^\prime$. The vector after value projection is $\alpha^{\prime\prime}$. Then $\pi_{pos}(\alpha^{\prime\prime}) = -\pi_{pos}(\alpha^{\prime})$, $\pi_{tmp_{pos_{sc}}}(\alpha^{\prime\prime}) = \pi_{pos_{sc}}(\alpha^{\prime})$, $\pi_{tmp_{q}}(\alpha^{\prime\prime}) = \pi_{q}(\alpha^{\prime})$, and, leaving rest to $0$. Consider $W_O^{(1)}$ being the identity, note that the residual connection in $\alpha^{\prime\prime} + y_t$ does not affect the non-zero values of $y_t$ in any of its blocks. The linear projection $W_1^{(1)}$ performs two operations. It transforms the block into $-\pi_{pos}(\alpha^{\prime\prime} + y_t)$ and fulfills the role of $W_1$ in \autoref{lem:shiftonehot}, where the blocks $tmp_q$, $h$, and $tmp_h$ are treated as placeholders for $m$, $h$, and $h^\prime$, respectively. After activation $\relu$, the block $pos$ will register $S\pi_{pos}(y_t)$. Followed by $W_2^{(1)}$, the block ${tmp_h}$ stores the one-hot encoding of $s_{\tau+1}$. Let the resulting representation be denoted by \(y_t^{(1)}\). It is important to emphasize that, with the exception of the block $flg$, all updates and value assignments within the constituent blocks of $y_t^{(1)}$ are carried out in auxiliary blocks. This ensures that the actual blocks associated with the preceding tokens in the layered computation remain unchanged throughout the intermediate computational stages. Since $flg$ does not participate in both the query and key computations of the remaining layers, it remains unaffected by these operations.

        The subsequent layers work by guessing that $q_{\tau+1}$ is a $\DD \ (-1)$ step. However, the last layer applies \autoref{lem:mMc}\footnote{On the corresponding tree of $n^\prime = \frac{n}{2}+1$ nodes, the Strahler number is at most $\log_2(n^\prime+2)$, so any $\lambda$ above this threshold satisfies the lemma.} to perform the final correction in $M$ and $c$ blocks depending on the value of the block $q$ in $y_t^{(4)}$. The second layer will realize the observation we made earlier in this proof {--} finding the appropriate token (among $y_i$s $i < t$) such that $\pi_{ht}(y_i) = \pi_{tmp_{ht}}(y_t^{(1)})$ and $\nexists \ i^\prime, i< i^\prime < t$ and $\pi_{ht}(y_{i^\prime}) = \pi_{tmp_{ht}}(y_t^{(1)})$. Let the query be $Q^{(2)}y_t^{(1)} = [\pi_{tmp_{ht}}(y_t^{(1)}) \parallel \pi_{flg}(y_t^{(1)})]$ and the key be $K^{(2)}\alpha = [\lambda\pi_{ht}(\alpha) \parallel \pi_{pos_{sc}}(\alpha)]$ for all $\alpha \in  \{x_0^{(1)}, x_1^{(1)}, \ldots, x_{n-1}^{(1)}, y_0^{(1)}, \ldots, y_t^{(1)}\}$ and $\lambda > n+1$. Let this vector be $\alpha^\prime$. Clearly $\pi_{ht}(y_t^{(1)})^\top\pi_{tmp_{ht}}(y_t^{(1)}) = 0$ ensures that we obtain a right-most $i < t$. The value projection $V^{(2)}$ then
        \[
            \begin{aligned}
                & \pi_{M_{buf}}(V^{(2)} \alpha^\prime) = \pi_M(\alpha^\prime), \ \  \pi_{c_{buf}}(V^{(2)} \alpha^\prime) = \pi_c(\alpha^\prime), \\
                & \pi_M(V^{(2)} \alpha^\prime) = 0, \quad
                \pi_c(V^{(2)} \alpha^\prime) = 0.
            \end{aligned}
        \]

        The subsequent operations will be analogous to that of \autoref{thm:cottreestrahler} that realizes \autoref{obs:dynamicstrdfs} using the blocks $M$, $c$, $M_{buf}$, $c_{buf}$, $tmp_{D-1}$, $tmp_{D}$, $tmp_{D+1}$, $tmp_{c_{buf}-1}$, $tmp_{c_{buf}}$, $tmp_{c_{buf}+1}$, and $flg$ in place of $flg_3$ of \autoref{thm:cottreestrahler}. However, in the fourth layer we employ \autoref{lem:mMc} that in its FFN block by dissolving the gating \eqref{eq:gatingTBstrahlertree} in $W_2^{(4)}$. This final projection $W_2^{(4)}$ produces $y_{t+1}$ by nullifying all temporary positions along with the buffers $M_{buf}$ and $c_{buf}$. Finally, $\str(\psi(w)) = \pi_M(y_{n}) + \ind[\pi_c(y_{n}) \ge 2]$.
    \end{proof}

    It should be noted that the pair $(t^\prime, s_{t^\prime})$ in \autoref{obs:strpathtwoMc} plays a pivotal role in transmitting the Strahler rank to the current vertex. In the absence of this contribution, the pair $(M_v, c_v)$ alone would not suffice to determine the Strahler rank at the current vertex (see \autoref{fig:previousMc}).
    \begin{figure}
        \centering
        \resizebox{0.75\linewidth}{!}{
            \begin{tikzpicture}
                \input{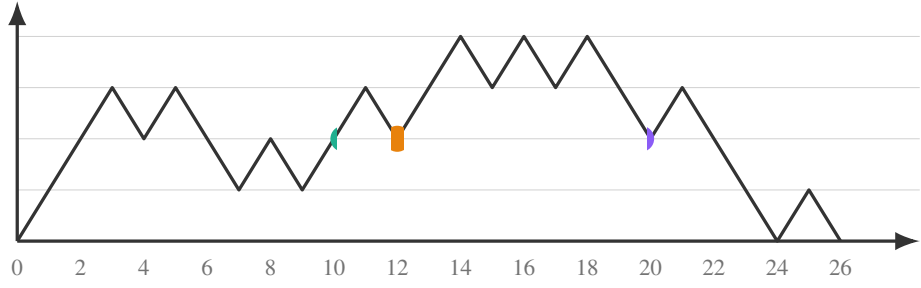}
            \end{tikzpicture}
        }
        \caption{As illustrated in \autoref{fig:pathfolding}, calculating the Strahler number at the point $(24,2)$ without incorporating the specific pair $(t^\prime, s_{t^\prime}) = (16,2)$ (from \autoref{obs:strpathtwoMc}) yields a value equal to that evaluated at $(20,2)$ in the above path. This would incorrectly suggest a generation-invariant property of the Strahler number for previously explored children.}
        \label{fig:previousMc}
    \end{figure}
    
\section{Realizing \texorpdfstring{$\wid$}{wid} on Trees and Paths}
\label{app:cotwid}
    \begin{lemma}
    \label{lem:widsupffn}
        Let $v = [v_1 \parallel v_2 \parallel v_3]^\top \in \R^3$ be an input vector such that $v_1 \in \{0, 1\}$, and $v_2, v_3 \in \mathbb{Z}^+$ are bounded such that $v_2 \le \lambda$ for a large $\lambda$. Let the target transformation $f: \R^3 \to \R^3$ be defined piecewise as:
        \[
            f(v) = \begin{cases} 
            [0 \parallel v_2 + 1 \parallel \max(v_3, v_2 + 1)]^\top \text{ if } v_1 = 0 \\ 
            [1 \parallel 1 \parallel \max(v_3, 1)]^\top \text{ if } v_1 = 1 
            \end{cases}
        \]
        It can be exactly computed by a single hidden-layer FFN utilizing a $\relu$ activation function with a hidden dimension of 4, taking the form $f(v) = W_2 \relu(W_1 v + b_1) + b_2$.
    \end{lemma}

    \begin{proof}
        To construct the exact FFN, let $\lambda \in \R$ be a strictly positive large scalar. 

        Define the weight matrix $W_1 \in \R^{4 \times 3}$ and bias vector $b_1 \in \R^4$:
        \[
            W_1 = \begin{pmatrix} 
            1 & 0 & 0 \\ 
            -\lambda & 1 & 0 \\ 
            0 & 0 & 1 \\ 
            -\lambda & 1 & -1 
            \end{pmatrix}, \quad 
            b_1 = [0 \parallel 0 \parallel 0 \parallel 1]^\top
        \]
        Evaluating the pre-activation state $z = W_1 v + b_1$ yields $z = [v_1 \parallel v_2 - \lambda v_1 \parallel v_3 \parallel v_2 - v_3 - \lambda v_1 + 1]^\top$.

        Let $h = \relu(z) = \max(0, z)$. Because $v_1 \in \{0, 1\}$ and $v_3 \ge 1$, the first and third components are strictly non-negative and pass through unchanged ($h_1 = v_1$, $h_3 = v_3$). To evaluate the conditional components for both valid states of $v_1$:

        \begin{description}
            \item Case $v_1 = 0$: The second term is $v_2$, and the fourth term is $v_2 - v_3 + 1$. Thus, the hidden state is $h = [0 \parallel v_2 \parallel v_3 \parallel \max(0, v_2 - v_3 + 1)]^\top$.
            \item Case $v_1 = 1$: The second term becomes $v_2 - \lambda$. Since $\lambda > \lambda \ge v_2$, this is strictly negative and clamps to zero. The fourth term becomes $v_2 - v_3 - \lambda + 1$. Since $v_2 - \lambda + 1 \le 0$ and $v_3 \ge 1$, this entire expression is strictly negative and also clamps to zero. Thus, the hidden state is $h = [1 \parallel 0 \parallel v_3 \parallel 0]^\top$.
        \end{description}

        The final projection $W_2 \in \R^{3 \times 4}$ and output bias $b_2 \in \R^3$:
        \[
            W_2 = \begin{pmatrix} 
            1 & 0 & 0 & 0 \\ 
            0 & 1 & 0 & 0 \\ 
            0 & 0 & 1 & 1 
        \end{pmatrix}, \quad
        b_2 = [0 \parallel 1 \parallel 0]^\top
        \]
        Applying this final projection verifies the transformation exactly maps to the conditions of $f(v)$:

        \begin{itemize}
            \item For $v_1 = 0$:
                \begin{align*}
                    f(v) &= \begin{pmatrix} 
                    1 & 0 & 0 & 0 \\ 
                    0 & 1 & 0 & 0 \\ 
                    0 & 0 & 1 & 1 
                    \end{pmatrix} \begin{bmatrix} 0 \\ v_2 \\ v_3 \\ \max(0, v_2 - v_3 + 1) \end{bmatrix} + \begin{bmatrix} 0 \\ 1 \\ 0 \end{bmatrix}\\
                    &= \begin{bmatrix} 0 \\ v_2 + 1 \\ v_3 + \max(0, v_2 - v_3 + 1) \end{bmatrix} = \begin{bmatrix} 0 \\ v_2 + 1 \\ \max(v_3, v_2 + 1) \end{bmatrix}.
                \end{align*}
            \item For $v_1 = 1$:
                \begin{align*}
                    f(v) &= \begin{pmatrix} 
                    1 & 0 & 0 & 0 \\ 
                    0 & 1 & 0 & 0 \\ 
                    0 & 0 & 1 & 1 
                    \end{pmatrix} \begin{bmatrix} 1 \\ 0 \\ v_3 \\ 0 \end{bmatrix} + \begin{bmatrix} 0 \\ 1 \\ 0 \end{bmatrix} = \begin{bmatrix} 1 \\ 0 + 1 \\ v_3 + 0 \end{bmatrix} = \begin{bmatrix} 1 \\ 1 \\ v_3 \end{bmatrix}
                \end{align*}
                Since $v_3 \in \mathbb{Z}^+ \implies v_3 \ge 1$, we can substitute $v_3 = \max(v_3, 1)$, yielding $[1 \parallel 1 \parallel \max(v_3, 1)]^\top$.
        \end{itemize}
        Note that, the first component of $f(v)$ can be even assigned to $0$, when $v_1 = 1$, by changing $W_2[0,0] = 0$.
    \end{proof}

    \begin{theorem}
    \label{thm:cotwidoftree}
        A three-layer single-head Transformer decoder can find the width of a tree $\wid(G_T)$ during the simulation of $\dij$ in exactly $n-1$ steps, where $n$ denotes the number of vertices in $G_T$.
    \end{theorem}
    \begin{proof}
        To accommodate the computation of $\wid(G_T)$, we augment the vectors $\tilde{x}_0, \ldots, \tilde{x}_{n-1}, \tilde{y}_0, \ldots, \tilde{y}_t$ as shown in \autoref{thm:cotdijkstra} with three scalar coordinates such that 
        \[
            x_i = [\tilde{x}_i^\top \parallel 0 \parallel 0 \parallel 0]^\top \text{ and } y_0 = [\tilde{y}_0^\top \parallel 0 \parallel 1 \parallel 1]^\top
        \]
        $\in \R^d$, where $d=5n+4$. The appended positions store, in order, the difference in distances $\pi_{crd}(\tilde{y}_{t})$ and $\pi_{crd}(\tilde{y}_{t-1})$, maps $twd^{(t)}$ and $mwd^{(t)}$. We will use the selectors $\pi_{dif}, \pi_{twd}, \text{ and } \pi_{mwd}: \R^d \to \R$ to select these positions along with reusing the previous selectors as in \autoref{thm:cotdijkstra} as and when necessary.

        The two layers of \autoref{thm:cotdijkstra} remain identical except for $W_2^{(2)}$; the dimensional extension does not affect the previous computation once the parameters are adapted accordingly. In addition to the designated changes $W_2^{(2)}$ does in ${y_t^{(1)}}^\prime$ (as in $\tilde{y}_t^{(1)^\prime}$), it will set $\pi_{dif}({y_t^{(1)}}^\prime)$ to $ \sum_{i=0}^{n-1}\pi_{buf}({y_t^{(1)}}^\prime)[i] - \pi_{crd}({y_t^{(1)}}^\prime)$ to obtain $y_t^{(2)}$. Note that, $\pi_{crd}({y_t^{(1)}}^\prime) = \pi_{crd}(y_t)$.
        
        The third layer contains a trivial attention layer that takes $W_O^{(3)}$ as empty. Since the tree $G_T$ is (un)weighted taking weights in $\{1, \infty\}$, and as we have already observed $\pi_{dif} \in \{0,1\}$ in \autoref{obs:dynamicwidbfs}, the subsequent FFN layer can be developed analogus to \autoref{lem:widsupffn} considering $V_{max} = n$. Consider this as the resultant vector $y_{t+1}$. Clearly, $\pi_{mwd}(y_{n-1}) = \wid(G_T)$. 
    \end{proof}

    \begin{lemma}
    \label{lem:addonetospecificheight}
        Let $v \in \R^{H+2}$ be a vector defined as the concatenation 
        $v = [m \parallel h \parallel p_1 \parallel p_2 \parallel \dots \parallel p_{H}]^\top$, 
        where $H > h \ge 1$ is a strictly bounding integer constant, $p_i \ge 1$ for all $i \in \{1, \dots, H\}$, and $m \in \{-1, 1\}$. 
        There exists an FFN with $\relu$ activation, such that:
        \begin{equation}
        \label{eq:addonetospecificheight}
            v^\prime = 
            \begin{cases} 
                \left[m \parallel h+1 \parallel p_1 \parallel \dots  \parallel p_{h+1} + 1 \parallel \dots \parallel p_{H}\right]^\top & \text{if } m = 1 \\
                \left[m \parallel h-1 \parallel p_1 \parallel p_2 \parallel \dots \parallel p_{H}\right]^\top & \text{if } m = -1
           \end{cases}
        \end{equation}
    \end{lemma}
    
    \begin{proof}
        Let the input dimension be $d_{in} = H + 2$. We construct a hidden layer of dimension $d_h = 4H$ to evaluate the conditional mapping. We partition the hidden dimension into two functional blocks: state preservation mappings and discrete pulse functions. The weight matrices $W_1 \in \R^{4H \times (H+2)}$ and $W_2 \in \R^{(H+2) \times 4H}$ are constructed by concatenating their respective sub-matrices.
        
        Since $m \in \{-1, 1\}$, we map it to the positive domain using two nodes: $\max(0, m)$ and $\max(0, -m)$. 
        The state preservation sub-matrix $W_{1,\text{id}} \in \R^{(H+3) \times (H+2)}$ is defined as:
        \begin{equation*}
            W_{1,\text{id}} = 
            \begin{pmatrix}
                1 & 0 & \mathbf{0}_{1 \times H} \\
                -1 & 0 & \mathbf{0}_{1 \times H} \\
                0 & 1 & \mathbf{0}_{1 \times H} \\
                \mathbf{0}_{H \times 1} & \mathbf{0}_{H \times 1} & I_{H}
            \end{pmatrix}
        \end{equation*}
        where $I_{H}$ is the identity matrix. Since $h \ge 1$ and $p_i \ge 1$, the $\relu$ acts as the identity function for all these dimensions.
        
        To selectively increment $p_{h+1}$ when $m=1$, we construct an indicator for each $k \in \{1, \dots, H-1\}$ that evaluates to $1$ when $m=1 \land h=k$, and $0$ otherwise. This requires $3$ nodes per index $k$ to form a piecewise linear basis. We stack $H-1$ blocks of size $3 \times (H+2)$ to form $W_{1,\text{pls}}$. The $k$-th block is:
        \begin{equation*}
            W_{1,\text{pls}}^{(k)} = 
            \begin{pmatrix}
                -(k-1) & 1 & \mathbf{0}_{1 \times H} \\
                -k & 1 & \mathbf{0}_{1 \times H} \\
                -(k+1) & 1 & \mathbf{0}_{1 \times H}
            \end{pmatrix}
        \end{equation*}

        The second layer matrix $W_2 = \begin{bmatrix} W_{2,\text{id}} & \parallel & W_{2,\text{pls}} \end{bmatrix}$ linearly combines the hidden representations to yield $v^\prime$. 
        
        First, $W_{2,\text{id}} \in \R^{(H+2) \times (H+3)}$ recovers the unmodified states and structurally applies the $m$ increment to $h$:
        \begin{equation*}
            W_{2,\text{id}} = 
                \begin{pmatrix}
                    1 & -1 & 0 & \mathbf{0}_{1 \times H} \\
                    1 & -1 & 1 & \mathbf{0}_{1 \times H} \\
                    \mathbf{0}_{H \times 1} & \mathbf{0}_{H \times 1} & \mathbf{0}_{H \times 1} & I_{H}
                \end{pmatrix}
        \end{equation*}
        Finally, $W_{2,\text{pls}} \in \R^{(H+2) \times 3(H-1)}$ routes the linear coefficients $[1, -2, 1]$ to the corresponding coordinates for $p_2$ through $p_{H}$ (noting $p_1$ is never modified since $h \ge 1 \implies h+1 \ge 2$):
        \begin{equation*}
            W_{2,\text{pls}} = 
            \begin{pmatrix}
                \mathbf{0}_{3 \times 3} & \mathbf{0}_{3 \times 3} & \dots & \mathbf{0}_{3 \times 3} \\
                [1 \parallel -2 \parallel 1] & \mathbf{0}_{1 \times 3} & \dots & \mathbf{0}_{1 \times 3} \\
                \mathbf{0}_{1 \times 3} & [1 \parallel -2 \parallel 1] & \dots & \mathbf{0}_{1 \times 3} \\
                \vdots & \vdots & \ddots & \vdots \\
                \mathbf{0}_{1 \times 3} & \mathbf{0}_{1 \times 3} & \dots & [1 \parallel -2 \parallel 1]
            \end{pmatrix}
        \end{equation*}
        
        By construction, if $m=1$, the $k$-th discrete pulse evaluates to $\max(0, h - k + 1) - 2\max(0, h - k) + \max(0, h - k - 1)$. Because $h$ is integral, this yields exactly $1$ at $h=k$ and $0$ otherwise. 
        Conversely, if $m=-1$, the arguments to the $\relu$s become $h + k - 1$, $h + k$, and $h + k + 1$. Because $h \ge 1$ and $k \ge 1$, we are guaranteed that $h + k - 1 \ge 1 > 0$. Consequently, all arguments are strictly positive, placing the $\relu$s purely in their linear regime. The combination perfectly collapses to $(h + k - 1) - 2(h + k) + (h + k + 1) = 0$. Thus, no pointer $p_i$ is affected when $m=-1$, satisfying the lemma.
    \end{proof}

    \begin{lemma}
    \label{lem:addwidplus1}
        Let $v \in \R^{H+1}$ be a vector defined as the concatenation $v = [p_1 \parallel p_2 \parallel \dots \parallel p_H \parallel wd]^\top$, where $wd \ge 0$ and $p_i \ge 0$ for all $i \in \{1, \dots, H\}$. Assume that $p_i \le wd + 1$ for all $i$, and that $p_i = wd + 1$ for at most one index $i$. There exists an FFN with ReLU activation, such that:
        \begin{equation}
        \label{eq:addwidplus1}
            v^\prime = 
            \begin{cases} 
                [p_1 \parallel p_2 \parallel \dots \parallel p_H \parallel wd + 1]^\top
                & \text{if } \exists i \text{ s.t. } p_i = wd + 1 \\
                v & \text{otherwise}
           \end{cases}
        \end{equation}
    \end{lemma}
    
    \begin{proof}
        Let the input dimension be $d_{in} = H + 1$. We construct a hidden layer of dimension $d_h = 2H + 1$. We partition the hidden layer into two sets of nodes: an identity mapping to preserve the input state, and a set of $H$ temporary indicator nodes to detect the increment condition.
        
        The first weight matrix $W_1 \in \R^{(2H+1) \times (H+1)}$ is constructed as a block matrix:
        \begin{equation*}
            W_1 = 
            \begin{pmatrix}
                I_{H+1} \\
                I_H  -\mathbf{1}_H
            \end{pmatrix}
        \end{equation*}
        where $I_k$ denotes the $k \times k$ identity matrix, and $\mathbf{1}_H \in \R^{H \times 1}$ is the all-ones column vector. 
        
        Applying the $\relu$ activation yields the hidden state $h = \max(0, W_1 v) \in \R^{2H+1}$. 
        Because $p_i \ge 0$ and $wd \ge 0$, the first $H+1$ dimensions pass through the $\relu$ unchanged, preserving the original vector $v$. 
        The remaining $H$ dimensions compute the temporary variables $tmp_i = \max(0, p_i - wd)$ for $i \in \{1, \dots, H\}$. 
        
        Given the constraint $p_i \le wd + 1$, the difference $p_i - wd$ is bounded above by $1$. Consequently, the temporary variables evaluate to:
        \begin{equation*}
            tmp_i = 
            \begin{cases} 
                1 & \text{if } p_i = wd + 1 \\
                0 & \text{if } p_i \le wd
            \end{cases}
        \end{equation*}
        Since $p_i = wd + 1$ holds for at most one index, the vector $tmp = [tmp_1 \parallel \dots \parallel tmp_H]^\top$ is guaranteed to be either a one-hot vector or the zero vector.
        
        The second weight matrix $W_2 \in \R^{(H+1) \times (2H+1)}$ reconstructs the output dimension and aggregates the temporary variables into the $wd$ coordinate:
        \begin{equation*}
            W_2 = 
            \begin{pmatrix}
                I_H & \mathbf{0}_{H \times 1} & \mathbf{0}_{H \times H} \\
                \mathbf{0}_{1 \times H} & 1 & \mathbf{1}_H^\top
            \end{pmatrix}
        \end{equation*}
        
        Multiplying $h$ by $W_2$ leaves the components $p_1 \dots p_H$ strictly isolated and unmodified. The final coordinate is computed as $wd^\prime = wd + \sum_{i=1}^H tmp_i$. Because $tmp$ contains at most a single $1$ corresponding to the condition $p_i = wd + 1$, the sum evaluates to $1$ if the condition is met, and $0$ otherwise.
    \end{proof}

    \begin{theorem}
    \label{thm:cotwidofpath}
        A two-layer single-head Transformer decoder can find $\wid(\psi(w))$ for a Dyck word $w$ in exactly $n$ CoT steps, where $|w| = n$.
    \end{theorem}
    \begin{proof}
        Given a Dyck word $w$ of length $n$, $\hgt(\psi(w))$ is bounded above by $\frac{n}{2}+1$, that is, when $\psi(w)$ is a path. The $t$\textsuperscript{th} decoding step records the width $\hgt(\psi(w))$ of the partial tree obtained after reading the prefix $w_0 \cdots w_{t-1}$, equivalently the lattice path $\{(\tau, s_\tau)\}_{\tau=0}^{t}$. The embedding dimension is $d = \frac{3n}{2}+4$, split into five blocks: the position $pos \in \R^{n}$ in one-hot form; the next position is a scalar holding the step $q_t \in \{-1,+1\}$ on an input token; the next is another scaler storing the cumulative sum of all $q_\tau$ see till the current CoT step. From the remaining block of length $\frac{n}{2}+1$, $p_i$ stores $|\{\tau \mid s_\tau = i \text{ and } q_\tau = s_{\tau} - s_{\tau-1} = +1\}|$ for $0 \le \tau \le t$ and $i \in \{1, \ldots, \frac{n}{2}\}$. And the final block $wd$ is a scalar that stores the maximum of all $p_i$, which turns out to be the width, as claimed in \autoref{prop:widthpathtree}.

        Assume $\mathcal{E}_n = \{e_0, \ldots, e_{n-1}\}$ denotes the standard basis vector on $\R^{n}$. The input tokens $x_0, \ldots, x_{n-1} \in R^d$ encodes the lattice path $\{\tau, s_\tau\}_{\tau=0}^{n}$ as follows:
        \[
            x_\tau = [e_\tau \parallel q_\tau \parallel 0 \parallel 0^{\frac{n}{2}+1} \parallel 0]^\top.
        \]
        The Dynamic is carried out by tokens $y_t \in \R^d$ with 
        \[
            y_0 = [e_0 \parallel 0 \parallel 0 \parallel 0^{\frac{n}{2}+1} \parallel 0]^\top.
        \]
        The selectors $\pi_{pos} : \R^d \to \R^n$ selects the first one-hot block that corresponds to the current position, $\pi_{q}: \R^d \to \{-1, 0, 1\}$ selects the block that records the corresponding move taken at the position for input tokens or $0$ for generated tokens, $\pi_{ht} : \R^d \to \R_{\ge 0}$ records the height, $\pi_{p_i}: \R^d \to \R_{\ge 0}$ records $p_i$ as discussed above. Finally, $\pi_{pos}: \R^d \to \R_{\ge 0}$ records the $\wid$. Analogous to \autoref{thm:cotpathstrahler}, the attention mechanism of the first layer employs a shift matrix $S$ to find the $\pi_{pos}$ and records the corresponding $\pi_{q}$ from the input tokens. Let $\alpha^\prime$ be the vector after multiplication by $W_O^{(1)}$ where every block is $0$ except $\pi_q(\alpha^\prime) = q_\tau$. After the residual connection with $y_t$, let this vector be $\tilde{y}_t$. We employ $W_1^{(1)}$ and $W_2^{(1)}$ of the following FFN layer similar to \autoref{lem:addonetospecificheight} so that $[\pi_q(\tilde{y}_t) \parallel \pi_{ht}(\tilde{y}_t) \parallel \pi_{p_1}(\tilde{y}_t) \parallel \cdots \parallel \pi_{\frac{n}{2}+1}(\tilde{y}_t)]^\top$ serves the vector $v$ and undergoes an update as in \eqref{eq:addonetospecificheight}. The remaining positions remain unchanged.

        Note that the above operation would impact only one position of all $p_1, \ldots, p_{\frac{n}{2}+1}$. For the second layer, we will utilize the FFN only. We will consider $w_O^{(1)}$ as $\mathbf{0}$. Let the vector obtained after applying the residual connection be $\tilde{y}_t^{(1)}$. The FFN layer performs updates on $[\pi_{p_1}(\tilde{y}_t^{(1)}) \parallel \cdots \parallel \pi_{\frac{n}{2}+1}(\tilde{y}_t^{(1)}) \parallel wd]^\top$ similar to \eqref{eq:addwidplus1} as shown in \autoref{lem:addwidplus1}. The linear projections do not affect any other positions of $\tilde{y}_t^{(1)}$ except that $W^{(2)}_2$ updates the the block $q$ to $0$. The resultant vector is $y_{t+1}$ and $\pi_{wd}(y_{n}) = \wid(\psi(w))$.
     \end{proof}

\end{document}